%% file: main.tex
\documentclass[11pt,letterpaper]{article}

\usepackage[top=1in,bottom=1in,left=1in,right=1in]{geometry}

\usepackage{times}
\usepackage{mathtools}
\usepackage{multirow}
\usepackage{amsmath}
\usepackage{amsthm}
\usepackage{amssymb}
\usepackage{bbm}
\usepackage{bm}
\usepackage{commath}
\usepackage{enumitem}
\setlist[itemize]{align=parleft,left=10pt}
\usepackage[utf8]{inputenc} 
\usepackage[T1]{fontenc}
\usepackage{hyperref}       
\usepackage{url}            
\usepackage{booktabs}       
\usepackage{amsfonts}  
\usepackage{nicefrac}       
\usepackage{microtype}      
\usepackage{xcolor}         
\usepackage{algorithm}
\usepackage{tcolorbox}
\usepackage{caption}
\usepackage{subcaption}
\usepackage[normalem]{ulem}
\newcommand\hl{\bgroup\markoverwith
    {\textcolor{yellow}{\rule[-.5ex]{.1pt}{2.5ex}}}\ULon}
\usepackage[noend]{algpseudocode}
\newcommand{\alglinenoNew}[1]{\newcounter{ALG@line@#1}}
\newcommand{\alglinenoPop}[1]{\setcounter{ALG@line}{\value{ALG@line@#1}}}
\newcommand{\alglinenoPush}[1]{\setcounter{ALG@line@#1}{\value{ALG@line}}}

\usepackage{lmodern}
\usepackage{mathtools}
\usepackage{multirow}

\usepackage{amsmath}
\usepackage{amsthm}
\usepackage{amssymb}
\usepackage{bbm}
\usepackage{bm}
\usepackage{commath}
\usepackage[font=small,labelfont=bf]{caption}   

\input{macro}

\usepackage[title,titletoc]{appendix}
\usepackage{color}

\allowdisplaybreaks

\author{
	\vspace{5pt}
  \textbf{Site Bai}, 
  \textbf{Chuyang Ke},  
  \textbf{Jean Honorio} \\
  Department of Computer Science\\
  Purdue University\\
  West Lafayette, Indiana, USA \\
  \texttt{\{bai123, cke, jhonorio\}@purdue.edu}\\
}
\date{}
\title{\huge Dual Convexified Convolutional Neural Networks}

\begin{document}
\maketitle
\begin{abstract}
We propose the framework of dual convexified convolutional neural networks (DCCNNs). 
In this framework, we first introduce a primal learning problem motivated by convexified convolutional neural networks (CCNNs), and then construct the dual convex training program through careful analysis of the Karush-Kuhn-Tucker (KKT) conditions and Fenchel conjugates. Our approach reduces the computational overhead of constructing a large kernel matrix and more importantly, eliminates the ambiguity of factorizing the matrix.
Due to the low-rank structure in CCNNs and the related subdifferential of nuclear norms, there is no closed-form expression to recover the primal solution from the dual solution. To overcome this, we propose a highly novel weight recovery algorithm, which takes the dual solution and the kernel information as the input, and recovers the linear weight and the output of convolutional layer, instead of weight parameter. 
Furthermore, our recovery algorithm exploits the low-rank structure and imposes a small number of filters indirectly, which reduces the parameter size.
As a result, DCCNNs inherit all the statistical benefits of CCNNs, while enjoying a more formal and efficient workflow.
\end{abstract}

\allowdisplaybreaks

\input{1_introduction}

\input{2_preliminaries}
\input{3_result}
\input{4_experiment}

\bibliographystyle{plain}
\bibliography{references}

\newpage
\appendix 
\input{5_supplement}

\end{document}

%% file: macro.tex
\usepackage{amsmath, amsthm, amssymb}
\usepackage{bm}
\usepackage{color}
\usepackage[normalem]{ulem}

\let\emph\textit

\newtheorem{theorem}{Theorem}

\newtheorem{lemma}{Lemma} 

\newtheorem{remark}{Remark}

\def\R{{\mathbb{R}}}

\DeclareMathOperator{\sign}{sign}
\DeclareMathOperator{\Tr}{Tr}
\DeclareMathOperator{\vecrize}{vec}
\DeclareMathOperator{\rank}{rank}
\DeclareMathOperator{\diag}{diag}
\newcommand{\argmin}{\operatorname*{arg\; min}}
\newcommand{\argmax}{\operatorname*{arg\; max}}

\newcommand{\st}{\operatorname*{subject\; to}}
\newcommand{\maximize}{\operatorname*{maximize}}
\newcommand{\minimize}{\operatorname*{minimize}}

\newcommand{\normnu}[1]{{\norm{#1}}_{\ast}}

%% file: 1_introduction.tex
\section{Introduction}

In the past decade, convolutional neural networks (CNNs) have become the cornerstone of various deep learning architectures. The performance of CNNs are so impressive, that they have become the standard methods on many tasks, including biomedical imaging \cite{suzuki2017overview, yamashita2018convolutional}, robot perception \cite{han2016incremental,zhang2019multi}, signal processing \cite{hershey2017cnn, eren2019generic}, among others. 
Despite the great success in a broad range of applications, the training of CNNs is in general an NP-hard problem \cite{blum1992training}. While in practice gradient-based optimization methods are used in the training process most of the time, there is no theoretical guarantee of achieving global optimality or bounding the rate of convergence.

To tackle the limitations of nonconvex landscapes, \cite{zhang2017convexified} proposes a class of convex relaxations of CNNs, namely the \emph{convexified convolutional neural networks} (CCNNs). 
Employing a convex optimization formulation, CCNNs can be optimized efficiently and their statistical properties can be analyzed rigorously \cite{zhang2017convexified}.
In CCNNs, an activation function is induced by a kernel function $\mathcal{K}$. The kernel function produces a kernel matrix $K$, which needs to be factorized into a matrix $Q$ such that $K = QQ^\top$.
After that, the CCNN algorithm solves a nuclear norm constrained convex optimization problem, and computes a low rank approximation, which outputs the weights of the network.
The whole convexified approach allows for efficient and optimal training of the network, and from a theoretical point of view makes precise statistical analysis viable \cite{zhang2017convexified}.

However on the algorithmic front, the CCNN framework, while enjoying the aforementioned benefits, does come with drawbacks and its own limitations.
\begin{itemize} 
    \item The factorization of the kernel matrix can be quite heuristic. One can either use Cholesky decomposition and obtain a triangular factor matrix $Q$, or take $Q = K^{1/2}$, or use the decomposition $K = UDU^\top$ and obtain $Q = UD^{1/2}$, to name a few. In the latter case, one can also use $Q = UD^{1/2}V^\top$ for any orthonormal $V$, meaning that there exists infinitely many possible $Q$'s! There is no guarantee that every choice of $Q$ performs equally well in the learning task, and thus the factorization problem itself can be tricky. Furthermore, the idea of factorizing the kernel matrix is not suitable from a statistical perspective, given that for CCNNs, kernels with infinitely dimensional basis functions are used (for example, Gaussian RBF kernel). 
    \item Solving the exact kernel factorization problem itself can be computationally expensive. Looking at Algorithm 1 in the CCNN framework \cite{zhang2017convexified}, the dimension of the kernel matrix is equal to the number of samples $n$ times the number of convolution operations (patches) $p$. From a space complexity perspective, the factorization step alone occupies memory in the order of $O(n^2 p^2)$. This is hardly feasible: for a small network training on $1000$ $28 \times 28$ images with stride 1, i.e. $784$ convolution operations, with $64$ bit double precision, the kernel matrix takes more than $4$ terabytes of memory! Similarly, the computational complexity of many factorization algorithms (for example, Cholesky) can take $O(n^3 p^3)$, which can be catastrophic in practice. 
    \item The CCNN algorithm heuristically cut the first $r$ columns of the convolutional weight recovered for the final low rank approximation step, with the number of filters $r$ as a hyperparameter. Setting the hyperparameter too large increases the parameter size with unnecessary information, and setting it too small voids the optimality guarantee. And such impact may accumulate as the number of layers increases. \vspace{-5pt}
\end{itemize}

In this paper, we are proposing a \emph{dual convexified convolutional neural network} (DCCNN) framework for learning CCNNs. We first carefully construct the DCCNN learning program by analyzing nuclear norms and Fenchel conjugates. We then solve the convex training problem in the derived dual form. Unfortunately, due to the nature of the sub-differential of the nuclear norm, there is no closed-form expression for us to recover the primal solution from the dual solution. To overcome this, we design a novel weight recovering algorithm that leverages the KKT optimality conditions and the sub-differential of nuclear norm. Without getting the primal solution, our algorithm directly recovers the linear weight and the convolution output using the dual solution and a kernel generating matrix, which applies to kernels with infinitely dimensional basis. 

Our DCCNN framework enjoys many benefits on top of the CCNN approach.
First, compared to the CCNN approach, our method does not require any factorization of the kernel matrix in the form of $K = QQ^\top$. Instead, by solving the optimization problem in DCCNN, we can directly recover the proper solution without introducing any ambiguity or compromise.
Next, our approach does not construct the full matrix $K$ for all samples in memory. Our weight recovering algorithm computes entries of the kernel matrix on the fly, greatly reducing the memory overhead.
More than that, our DCCNN framework does not require specifying the low rank hyperparameter manually. Our weight recovery algorithm automatically computes the minimum necessary rank from the dual solution, which better exploits the benefits of the low-rankness introduced by the nuclear norm constraint, as it encourages a small number of filters indirectly and reduces the parameter size.
To sum up, our DCCNN method inherits all the benefits of convexified CNNs like generalization error and sample complexity bounds, while enjoying a much more formal and efficient workflow without introducing ambiguity.

\textbf{Related works.}
The area of convex convolutional neural networks is fairly new. 
In \cite{zhang2017convexified}, it is shown that a class of CNNs with reproducing kernel Hilbert space (RKHS) filters contains a low-rank matrix structure, which can be further convexified by using a nuclear norm constraint and reduced to the class of CCNNs. This provides a basis for our work. 
Other convex neural network researches are more tangential.
\cite{mairal2014convolutional} proposed to approximate the CNNs with a translation-invariant kernel.
\cite{zhong2017learning} analyzed the convergence properties of CNNs. 
\cite{gunasekar2018implicit} analyzes a linear version of convolutional networks. 
\cite{amos2017input} and \cite{makkuva2020optimal} study the so-called input convex neural networks, where the output is a convex function of the inputs.
\cite{arora2019exact} studies the exact computation of CNNs with infinite many filters. \cite{ergen2020implicit} proposes equivalent convex regularizers to the CNN architectures. See Appendix \ref{morerelatedwork} for further discussion.

\textbf{Summary of our contribution.} We provide the following series of novel results in this paper:
\begin{itemize}
    \item We provide a dual convexified convolutional neural network (DCCNN) framework for learning a class of CNNs. We rigorously construct a dual convex training program through careful analysis of the KKT conditions and Fenchel conjugates. With the dual solution, we also propose a linear weight and implicit convolutional weight recovery algorithm.
    \item Our DCCNN framework is theoretically formal. Compared to prior literature, our method does not require any ambiguous factorization of the kernel matrix. Through the novel weight recovery algorithms, we directly recover the proper solution. In addition, our approach implicitly encourages a small number of filters, reducing the number of weight parameters without enforcing the low-rankness by heuristics.  
    \item Our DCCNN framework is highly efficient. It operates without loading the full kernel matrix for all samples in memory, or introducing any unnecessary memory overhead of factorization.
    As a result, DCCNN inherits all the statistical benefits of convexified CNNs, while enjoying a much more efficient workflow.  
    \item Our analysis in the DCCNN framework is novel. To derive the dual training program, the nuclear norm in the primal problem requires a careful construction of the Lagrangian through Fenchel conjugates. Furthermore, there is no closed-form expression that can be used to recover the primal solution from the dual solution directly because of the sub-differential of the nuclear norm. Without getting the primal solution, our algorithm directly recovers the linear weight and the convolution output using the dual solution and a kernel generating matrix, which applies to kernels with infinitely dimensional basis.
\end{itemize}

%% file: 2_preliminaries.tex
\section{Preliminaries}
\subsection{Notation} 
For any positive integer $n$, we use $[n]$ to represent the set $\{1,2,...,n\}$. We use regular lower-case letters (e.g., $a$, $c$) to denote constant. We use bold lower-case letters like $\bf{x}$ to represent vectors. We use capital letters (e.g., $U$, $V$) to represent matrix. We use $\rho$ as the notation for activation function, $\sigma$ as the matrix singular values, and $\lambda$ as matrix eigenvalues. $\Vert \cdot \Vert_\ast$ denotes the matrix nuclear norm, and $\Vert \cdot \Vert_2$ represents the spectral norm for a matrix, and the Euclidean norm of a vector. We use $\Tr(\cdot)$ to represent the trace of a matrix. For matrix $V$, we use $V(i)$ to represent the $i^{th}$ column vector and $V_{ij}$ to represent the entry in the $i^{th}$ row and $j^{th}$ column. For column vectors $\mathbf{x}_i \in \mathbb{R}^{d}$ with $i \in [n]$, we denote $\big[\mathbf{x}_i\big] \in \mathbb{R}^{d\times n}$ as the concatenation of these column vectors into a matrix. For matrix $U$ and $V$ with the same number of rows, $[U \mid V]$ represents the concatenation of the matrices. We use $\mathbf{I}$ to represent the identity matrix and $\mathbf{0}$ to represent the zero matrix.

\subsection{Mathematical Formulation of Convolutional Neural Networks}
Given a dataset $\{ (\mathbf{x}_1, y_1), ... , (\mathbf{x}_n, y_n) \}$ containing vector data $\mathbf{x}_i \in \R^{d_0}$ and target $y_i$, $i \in [n]$, a two-layer CNN, also called one-hidden-layer CNN, is a function that maps the data to the targets by passing the data through a convolutional layer and a linear layer. 

The convolutional layer performs convolution operations, in which a subset of $d_1$ entries in $\mathbf{x}_i$ is multiplied element-wise with one of the $r$ convolutional filters $\mathbf{w}_k \in \R^{d_1}$, $k \in [r]$. We call this subset of data entries a \textit{patch} generated from $\mathbf{x}_i$, denoted as $\mathbf{z} = z\left(\mathbf{x}_i\right) \in \R^{d_1}$. The number of patches relies on the choice of filter width, padding type, stride, etc., and we use $p$ to denote the total number of patches generated by a sample. The $p$ patches form the patch matrix $Z = \left[ \mathbf{z}_1, ... , \mathbf{z}_p \right] \in \R^{d_1 \times p}$, and the $r$ convolutional filters form the weight matrix $W = \left[ \mathbf{w}_1, ... , \mathbf{w}_r \right] \in \R^{d_1 \times r}$. Then the convolutional operations can be represented by the matrix multiplication between the patch matrix and weight matrix, i.e., $Z^\top W$. 

The result of the convolutional layer is then passed into an element-wise activation function $\rho\left( \cdot \right)$. After that, $\rho(Z^\top W)$ is passed into a linear layer. We denote the weight of the linear layer as $L \in \R^{p \times r}$. For all $j \in [p]$ and $k \in [r]$, $L_{jk}$ is the coefficient multiplied with the $k^{th}$ convolutional filter and the $j^{th}$ patch. Then the output of the two-layer CNN can be represented by $\Tr \big( \rho( Z^\top W ) L^\top \big)$. 

\subsection{Convexified Convolutional Neural Networks}
As proposed by \cite{zhang2017convexified} and later generalized by \cite{bietti2019group}, CNNs with certain choices of activation function can be contained in some reproducing kernel Hilbert space (RKHS). In other words, for some specific activation function $\rho$, there is a corresponding kernel function $\mathcal{K}\left(\cdot, \cdot \right) = \phi\left( \cdot \right)^\top\phi\left( \cdot \right)$ that produces similar non-linearity. The basis function $\phi\left( \cdot \right): \R^{d_1} \rightarrow \R^{d_2}$ could be a mapping to infinite dimensions, i.e., $d_2$ could be infinity. Define basis function matrix 
\begin{align}
    \Phi \left(\mathbf{x}_i\right) = \big[ \phi\big(z_1\left(\mathbf{x}_i\right)\big), \cdots, \phi\big(z_p\left(\mathbf{x}_i\right)\big) \big] \in \R^{d_2 \times p}, \nonumber
\end{align}
and let $W \in \R^{d_2 \times r}$, then the activation function is replaced by the kernel basis function, which leads to a convexified two-layer CNN function: 
\begin{align}
f(\mathbf{x}_i) = \Tr \big( \Phi (\mathbf{x}_i)^\top W L^\top \big). \label{DCCNNeq}
\end{align} 
In other words, activating the convolution output $\rho( Z^\top W )$ is equivalent to passing the data through the kernel basis function, i.e., $\Phi (\mathbf{x}_i)^\top W$. Define parameter matrix $A = W L^\top \in \R^{d_2 \times p}$. Then we can write the two-layer CNN using the new parameter $A$: 
\begin{align}
f(\mathbf{x}_i) = \Tr \big( \Phi (\mathbf{x}_i)^\top A \big). \label{binary_classifier} 
\end{align}
As the multiplication of the convolutional weight and linear weight matrices, $A$ is intrinsically low-rank, i.e., $\rank(A) \leq r$. \cite{zhang2017convexified} enforces the low-rankness by an extra relaxed constraint to bound the nuclear norm of $A$ by some constant $a$, i.e., $\Vert A \Vert_\ast \leq a$. 

In this paper we consider the task of learning two-layer CNN functions for classification with some loss $\ell(\cdot)$ that is convex and non-increasing, e.g. hinge loss, logistic loss, exponential loss. 

Using the convexified CNN function defined in Eq. \eqref{binary_classifier}, the parameter matrix $A$ can be learned by solving the following optimization problem. For binary classification, 
\begin{align}
\hat{A} &= \argmin_{A} \left \Vert A \right \Vert_\ast + c \sum_{i=1}^{n} \ell \big( y_i \Tr \big( \Phi (\mathbf{x}_i)^\top A \big) \big) \label{hinge_loss}
\end{align}
where $c>0$ is a hyperparameter. Furthermore, in the context of $m$-class classification, there is a predictor for each class, i.e., $f_k(\mathbf{x}_i) = \Tr \big( \Phi (\mathbf{x}_i)^\top A_k \big)$ for every $k \in [m]$. Correspondingly, the parameter $A = [A_1, ..., A_m]$ can be learned by minimizing the loss for multi-class classification: 
\begin{align} 
\hat{A} &= \argmin_{A} \left \Vert A \right \Vert_\ast + c \sum_{i=1}^{n} \sum_{k\neq y_i} \ell \big( f_{y_i}(\mathbf{x}_i) - f_{k}(\mathbf{x}_i) \big). \label{multi_hinge_loss}
\end{align}
It is worth noticing that the $\Phi (\mathbf{x}_i)$ appearing in the problem can be infinitely dimensional, and can only be represented by some heuristic finite approximation in practice, which consequently introduces ambiguity. To tackle this problem, we consider the dual formulation of the problem in this paper.

%% file: 3_result.tex
\begin{figure*}[t]
\centering
     \begin{subfigure}[b]{0.49\textwidth}
         \centering
         \begin{tcolorbox}[colback=white]
         \hspace{-12.5pt}\includegraphics[width=1.05\textwidth]{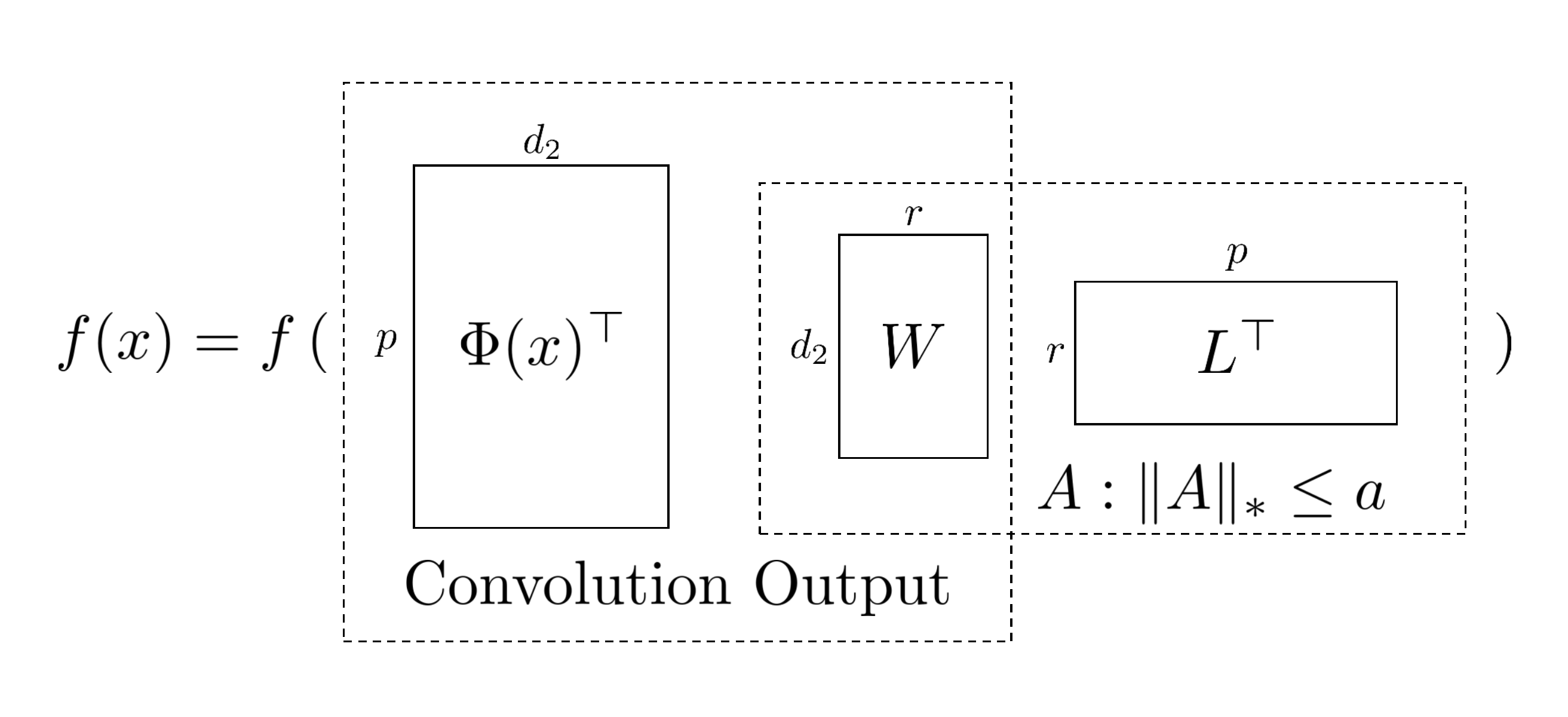}
         \begin{itemize}[noitemsep, leftmargin=15pt]
             \item \hspace{-15pt} \footnotesize $\Phi(x)$ approximated by kernel matrix factorization
             \item \hspace{-15pt} \footnotesize $A$ is low-rank, enforced by nuclear norm constraint
             \item \hspace{-15pt} \footnotesize $W$ computed by low-rank approximation
         \end{itemize}
         \end{tcolorbox}
         \caption{The Primal Framework}
         \label{fig:y equals x}
     \end{subfigure}
     \hfill
     \begin{subfigure}[b]{0.49\textwidth}
         \centering
         \begin{tcolorbox}[colback=white]
         \hspace{-12.5pt} \includegraphics[width=1.05\textwidth]{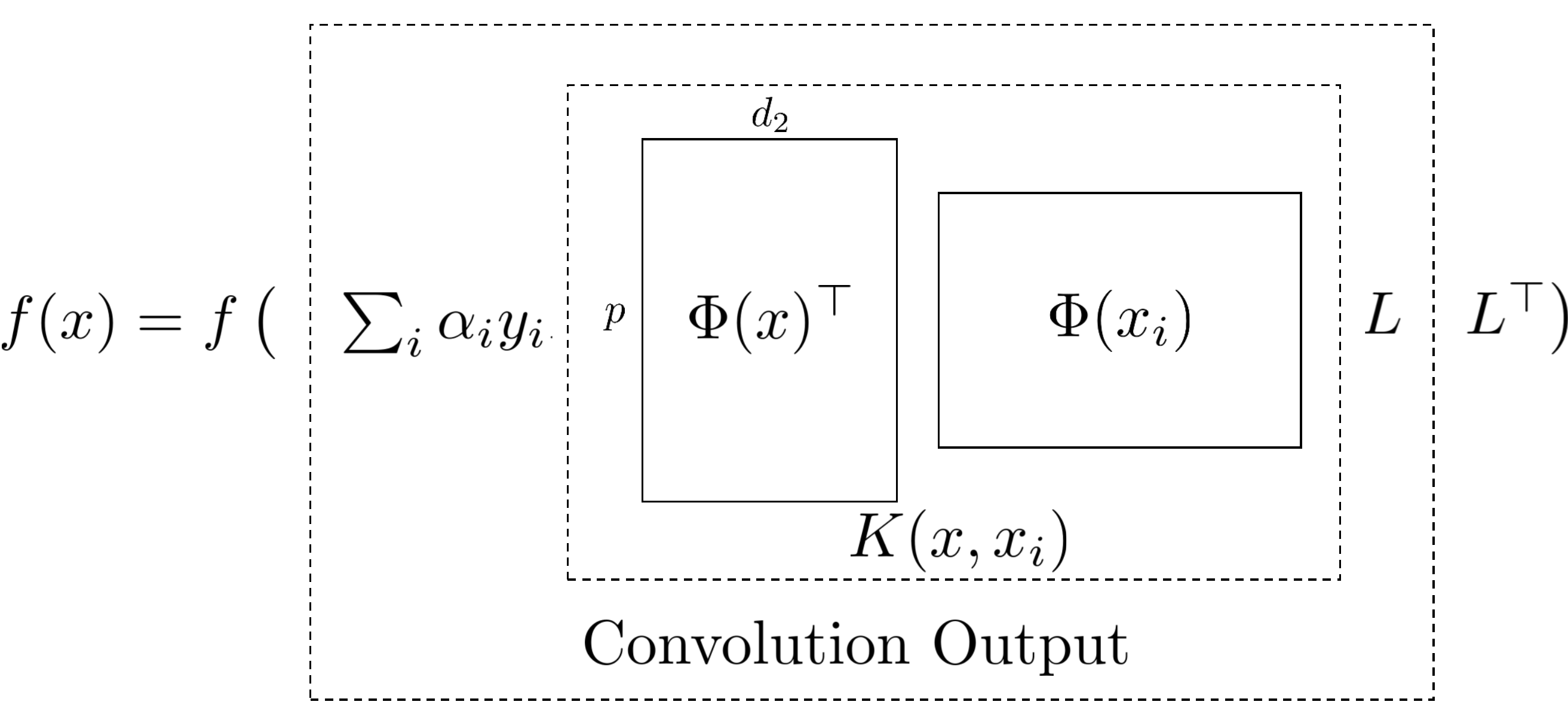}
         \begin{itemize}[noitemsep, leftmargin=15pt]
             \item \hspace{-15pt} \footnotesize Use $K(x,x_i)$ without ambiguous factorization
             \item \hspace{-15pt} \footnotesize Recover weight with dual solution $\alpha$
             \item \hspace{-15pt} \footnotesize Compute convolution output without $\Phi(x)$ or $W$
         \end{itemize}
         \end{tcolorbox}
         \caption{The Dual Framework}
         \label{fig:three sin x}
     \end{subfigure}
\caption{(a): In the primal framework, basis function matrix $\Phi(x)$ is approximated by a matrix $Q$ from the factorization of kernel matrix such that $K = QQ^\top$. The convolutional weight $W$ and linear weight $L$ are multiplied together as matrix $A$ with low-rankness enforced by nuclear norm constraint. $W$ is recovered by a low-rank approximation from optimized $A$. (b): The dual framework uses $K(x,x_i)$ without ambiguous factorization, and recovers the weights with the optimized dual variable $\alpha$. The primal solution $A$ cannot be directly recovered because $A$ has no closed-form expression of $\alpha$. Therefore, the dual framework recovers linear weight $L$ and computes the convolution output $\Phi(x)^\top W$ directly without $W$ or $\Phi(x)$.}
\label{frameworks}
\end{figure*}

\section{Main Results} \label{main_res}
We begin by introducing the primal optimization problem. Learning a two-layer convexified CNN, i.e., solving the loss minimization in Eq. \eqref{hinge_loss} is equivalent to solving the following optimization problem:
\begin{align} 
\minimize_{A} 
\quad & 
\Vert A \Vert_\ast + c \sum_{i = 1}^n \ell(\xi_i) \label{primal} \\
\st \quad & 
y_i\Tr \big( \Phi (\mathbf{x}_i)^\top A \big) \geq \xi_i \,,  
\ \forall i \in [n] \,,  \nonumber
\end{align}
in which $\xi_i$ is an equivalent variable of $y_i\Tr \big( \Phi (\mathbf{x}_i)^\top A \big)$, and $c > 0$ is some hyperparameter.  $\ell(\cdot)$ is some convex and non-increasing loss function. In fact, many common loss functions fall into this category, e.g. hinge loss $\ell_H(x) = \max(0, 1-x)$, squared hinge loss $\ell_{SH}(x) = (\max(0, 1-x))^2$, logistic loss $\ell_{L}(x) = \log(1+e^x)$, and exponential loss $\ell_{E}(x) = e^x$. 

In the primal problem in Eq. \eqref{primal}, the computation of the basis function matrix $\Phi \left(\mathbf{x}_i\right)$ may be infeasible, or otherwise heuristically approximated by some kernel matrix factorization, as the basis function $\phi\left( {x}_i \right)$ could be a mapping to infinite dimensions for many kernels. To avoid such dilemma, we first derive the dual problem, then introduce the algorithm for recovering the weight from the dual solution. An illustration of the dual framework based on the primal framework is demonstrated in Figure \ref{frameworks}. For the theorems in this section, we provide proof sketches and leave the detailed derivations in Appendix \ref{bi_proof}.

\subsection{Dual Optimization Problem} \label{dual_derivation}
  In order to avoid the ambiguous approximation from kernel matrix factorization, we consider the dual form of problem \eqref{primal} and compute a kernel generating matrix directly with the benefits of the kernel trick \cite{scholkopf2002learning}. 
\begin{theorem} \label{dual_theorem}
The dual problem of Eq. \eqref{primal} is given by: 
\begin{align} 
\maximize_{\alpha} 
\quad & 
-c\sum_{i = 1}^n \ell^\ast\left(-\frac{\alpha_i}{c}\right) \label{real_dual} \\
\st \quad & 
\lambda_{max} \Big(  \sum_{i = 1}^n \sum_{j = 1}^n \alpha_i \alpha_j y_i y_j K(\mathbf{x}_i, \mathbf{x}_j) \Big) \leq 1 \,, \nonumber  \\ 
& \alpha_i \geq 0 \,, 
\quad \forall i \in [n] \,, \nonumber  
\end{align} 
in which $\alpha_i$'s are the dual variables, $\ell^\ast(\cdot)$\footnote{To have a detailed illustration of $\ell^\ast(\cdot)$, we include the Fenchel conjugate of some common losses in Appendix \ref{app_conjugate}.} is the Fenchel conjugate of the loss function $\ell(\cdot)$, $K\left(\mathbf{x}_i, \mathbf{x}_j\right) = \Phi \left(\mathbf{x}_i\right)^\top \Phi \left(\mathbf{x}_j\right)$ is a kernel generating matrix, such that its entries are composed of the kernel function $\mathcal{K} \left( \cdot, \cdot \right) = \phi \left( \cdot \right)^\top\phi \left( \cdot \right)$ taking one patch from each of the two samples as input:  
\begin{align}
&K\left(\mathbf{x}_i, \mathbf{x}_j\right) = 
\begin{pmatrix}
\mathcal{K}\big(z_1(\mathbf{x}_i), z_1(\mathbf{x}_j)\big) & \cdots & \mathcal{K}\big(z_1(\mathbf{x}_i), z_p(\mathbf{x}_j)\big) \\
\mathcal{K}\big(z_2(\mathbf{x}_i), z_1(\mathbf{x}_j)\big) & \cdots & \mathcal{K}\big(z_2(\mathbf{x}_i), z_p(\mathbf{x}_j)\big) \\
\vdots  & \ddots & \vdots \\
\mathcal{K}\big(z_p(\mathbf{x}_i), z_1(\mathbf{x}_j)\big) & \cdots & \mathcal{K}\big(z_p(\mathbf{x}_i), z_p(\mathbf{x}_j)\big)
\end{pmatrix}. \label{kernel_matrix} 
\end{align}
\end{theorem}
\textit{Proof Sketch:} The derivation of the dual problem follows the standard Lagrangian duality framework, with a careful construction of the Fenchel conjugate \cite{boyd2004convex} for the nuclear norm and the loss function.
For Lagrangian variables $\alpha_i \geq 0$, we can write the Lagrangian function of the primal problem in Eq. \eqref{primal} as follows, after rearranging terms and utilizing the trace definition of matrix inner product:
\begin{align}
\mathcal{L}\left(A, \xi, \alpha\right) &= \left \Vert A \right \Vert_\ast - \big \langle \sum_{i = 1}^n \alpha_i y_i \Phi \left(\mathbf{x}_i\right), A \big \rangle + c\sum_{i = 1}^n \ell(\xi_i) + \sum_{i = 1}^n \alpha_i\xi_i.  \label{Lagrangian}
\end{align}
Then by definition, we get the dual function by minimizing the Lagrangian function with respect to the primal variables: 
\begin{align}
g\left(\alpha \right) &= \min_{A, \xi} \mathcal{L}\left(A, \xi, \alpha \right) \nonumber \\
&= - \underbrace{\max_{A} \big \langle \sum_{i = 1}^n \alpha_i y_i \Phi \left(\mathbf{x}_i\right), A \big \rangle - \left \Vert A \right \Vert_\ast}_{g_1} - \underbrace{\max_{\xi} - c\sum_{i = 1}^n \ell(\xi_i) - \sum_{i = 1}^n \alpha_i\xi_i}_{g_2},
\end{align}
in which $g_1$ and $g_2$ take the form of the Fenchel conjugate. The conjugate function of norms is their dual norm \cite{boyd2004convex}, then for the norm $f_0 = \left \Vert A \right \Vert_\ast$, we have
\begin{align} \label{conjugate}
& g_1 = f_0^\ast \Big( \sum_{i = 1}^n \alpha_i y_i \Phi \left(\mathbf{x}_i\right) \Big), \textit{in which} \ f_0^\ast \left( \cdot \right) = 
\begin{cases}
\ 0 \ , & \left \Vert \cdot \right \Vert_2 \leq 1 \\
\infty, & otherwise
\end{cases}.
\end{align}
Also, 
\begin{align}
    g_2 = c\sum_{i=1}^n \ell^\ast\left(- \frac{\alpha_i}{c}\right), \nonumber
\end{align}
in which $\ell^\ast(\cdot) = \sup_{\xi_i} \left\{ \langle \cdot, \xi_i \rangle - \ell(\xi_i) \right\}$ is the conjugate function of the loss function $\ell(\xi_i)$. Therefore, the dual problem of Eq. \eqref{primal} is as follows: 
\begin{align} 
\maximize_{\alpha} 
\quad & 
-c\sum_{i = 1}^n \ell^\ast\left(-\frac{\alpha_i}{c}\right) \nonumber \\
\st \quad & 
\big \Vert \sum_{i = 1}^n \alpha_i y_i \Phi \left(\mathbf{x}_i\right) \big \Vert_2 \leq 1 \,,  \label{dual} \\
& \alpha_i \geq 0 \,, 
\quad \forall i \in [n] \,. \nonumber
\end{align} 
Finally, we construct the kernel function by squaring the spectral norm constraint in Eq. \eqref{dual} on both sides, which leads to the form in Eq. \eqref{real_dual}. \hfill $\blacksquare$

The dual problem presents in the form of the kernel generating matrix and eliminates the basis function, and consequently the ambiguity of factorizing the kernel matrix. Also, our kernel generating matrix is of size $\mathcal{O}(p^2)$, which is more computationally efficient in terms of spatial complexity.

With the dual problem constructed, we solve the dual problem by a coordinate-descent algorithm. The algorithm picks the coordinate entry in ascending order with respect to the maximum eigenvalues of the kernel generating matrices, and optimize the chosen coordinate by binary search. The psuedocode of the algorithm is provided in Algorithm \ref{cd_alg} in Appendix \ref{cd_algorithm}.

\subsection{Recovering the Parameters} \label{binary_primal_recover}
In conventional machine learning dual optimization problems like the dual of Support Vector Machine (SVM), the primal solution can be exactly recovered with the stationary condition in the Karush–Kuhn–Tucker (KKT) conditions \cite{weston1998multi, Vapnik1998}. Such closed form expression for the primal solution, however, no longer applies to the DCCNN formulation, because the nuclear norm is non-differentiable. To tackle this problem, we propose a parameter recover algorithm that avoids recovering the primal solution $\hat{A}$, but directly recovers the linear weight $\hat{L}$ and the output of the convolutional layer $\Phi(x)^\top\hat{W}$ without explicitly recovering $\hat{W}$ or computing $\Phi(x)$. Our algorithm leverages the KKT conditions and the subdifferential set of the nuclear norm. 

Since there are infinite many ways to decompose the convolutional weight $\hat{W}$ and linear weight $\hat{L}$ from the optimal parameter $\hat{A}$, we adopt the singular value decomposition method proposed by \cite{zhang2017convexified}, i.e., for the compact SVD $\hat{A} = \hat{U}_1\hat{D}_1\hat{V}_1$, regard $\hat{U}_1$ as the convolutional weight $\hat{W}$ and $\hat{V}_1$ as the linear weight $\hat{L}$. We use them interchangeably in the following context. For recovering such weight parameters, we propose an algorithm to recover the linear weight in Theorem \ref{theo_binary_linear}, and we propose the method to implicitly recover the convolutional weight by the convolution output in Theorem \ref{theo_binary_conv}. The complete algorithm workflow is demonstrated in Algorithm \ref{recover_alg}.

\begin{algorithm*}[t]
\caption{Recovering the weight parameters for two-layer DCCNNs}
\label{recover_alg}
\begin{algorithmic}
\State {\bfseries Input:} Data $\{(\mathbf{x}_i, y_i)\}_{i=1}^n$; Optimized dual solution $\{\hat{\alpha}_i\}_{i=1}^n$ to problem \eqref{real_dual}; Kernel function $\mathcal{K}$. 
\State {\bfseries Recover the Linear Weight:}
\end{algorithmic}
\begin{algorithmic}[1]
    \alglinenoNew{alg1}
    \State Compute $S = \sum_{i, j} \hat{\alpha}_i \hat{\alpha}_j y_i y_j K \left(\mathbf{x}_i, \mathbf{x}_j\right)$. \Comment{Using $K(\cdot, \cdot)$ defined with $\mathcal{K}$ in Eq. \eqref{kernel_matrix} without ambiguous factorization}
    \State Compute the eigendecomposition $S = \tilde{V}\Lambda\tilde{V}^{-1}$.
    \State Let the eigenvectors with eigenvalue $1$ form the linear weight $\hat{L}$, i.e., $$\hat{L} = \left[\tilde{V}(i)\right] \in \R^{p \times r}, \text{ for all } i \text{ such that } \Lambda_{ii} = 1.$$
    \alglinenoPush{alg1}
\end{algorithmic}   
\begin{algorithmic}
\State {\bfseries Recover the Convolution Output (Input for the next layerwise training):}  \Comment{Without computing $\Phi(\mathbf{x}_i)$ or $W$}
\end{algorithmic}
\begin{algorithmic}[1]
    \alglinenoPop{alg1}
    \State Compute the output of the convolutional layer $\Phi(\mathbf{x}_i)^\top\hat{W}$ by $$\Phi(\mathbf{x}_i)^\top\hat{W} = \sum\nolimits_{j = 1}^n \hat{\alpha}_j y_j K (\mathbf{x}_i, \mathbf{x}_j) \hat{L}.$$ 
    \alglinenoPush{alg1}
\end{algorithmic}
\begin{algorithmic}
\State {\bfseries Output:} Linear weight $\hat{L}$; Output of convolutional layer $\{\Phi(\mathbf{x}_i)^\top\hat{W}\}_{i=1}^n$.
\end{algorithmic}
\end{algorithm*}

\subsubsection{Recovering the Linear Weight}
Our weight recovering algorithm leverages the KKT conditions, in which the subdifferential of the nuclear norm plays an important part.  We begin by introducing two Lemmas that formalize the stationary condition of our optimization problem and the relevant definition for the nuclear norm subdifferential.
\begin{lemma} \label{stationary_lemma}
For the primal problem \eqref{primal} and the dual variables $\{\alpha_i\}_{i=1}^n$, the stationary condition with respect to variable $A$ in the KKT conditions is given by:
\begin{align}
\partial \Vert \hat{A} \Vert_\ast - \sum_{i = 1}^n \hat{\alpha}_i y_i \Phi \left(\mathbf{x}_i \right) = 0. \label{stationary}
\end{align}
\end{lemma}
\begin{lemma}[Subdifferential of Nuclear Norm \cite{watson1992characterization}] \label{subdifferntial}
For matrix $A \in \R^{d_2 \times p}$ with $\rank(A) = r$, consider its compact SVD $A = U_1D_1V_1^\top$ in which $U_1 \in \R^{d_2 \times r}$, $D_1 \in \R^{r \times r}$, and $V_1 \in \R^{p \times r}$, and full SVD $A = UDV^\top$ in which $U \in \R^{d_2 \times d_2}$, $D \in \R^{d_2 \times p}$, and $V \in \R^{p \times p}$. Denote $U_2 \in \R^{d_2 \times (d_2-r)}$, $V_2 \in \R^{p \times (p - r)}$ such that $U = [U_1 \mid U_2]$, $V = [V_1 \mid V_2]$. Then the subdifferential set of $A$ is given by 
\begin{align}
\partial \left \Vert A \right \Vert_\ast = \Big\{ &U_1V_1^\top+U_2 E V_2^\top: \ E \in \R^{(d_2-r) \times (p-r)}, \ \sigma_{max}\left(E \right) \leq 1 \Big\}, \label{def2} 
\end{align}
in which $\sigma_{max}\left( \cdot \right)$ is the largest singular value.
\end{lemma}
Now we introduce our method of recovering the linear weight using the dual solution and the kernel generating matrix, and validate that what we recover is indeed the linear weight $\hat{V}_1$.

\begin{theorem} \label{theo_binary_linear}
Given the optimal dual solution $\{\hat{\alpha}_i\}_{i=1}^n$ to the problem in Eq. \eqref{real_dual}, let $\tilde{V} \in \R^{p \times p}$ be the matrix of eigenvectors from the eigendecomposition $\tilde{V} \Lambda \tilde{V}^{-1} = \sum_{i = 1}^n \sum_{j = 1}^n \hat{\alpha}_i \hat{\alpha}_j y_i y_j K(\mathbf{x}_i,\mathbf{x}_j)$, where $\Lambda = \diag(\lambda_1, ..., \lambda_p) \in \R^{p \times p}$. The linear weight $\hat{V}_1$ can be recovered by the eigenvectors in $\tilde{V}$ corresponding to the eigenvalue of 1, that is,
\begin{align}
\hat{V}_1 = \left[\tilde{V}(i)\right] \in \R^{p \times r}, \text{ for all } i \text{ such that } \lambda_i = 1, \label{linear}
\end{align}
in which $\tilde{V}(i)$ denotes the $i^{th}$ column of $\tilde{V}$ and $[\cdot]$ denotes the concatenation of these columns.
\end{theorem}

\textit{Proof Sketch:} By the stationary condition in Lemma \ref{stationary_lemma}, we know that the there exists one element in the subdifferential of $\Vert A \Vert$ at value $\hat{A}$ that satisfies $\partial \Vert \hat{A} \Vert_\ast = \sum_{i = 1}^n \hat{\alpha}_i y_i \Phi \left(\mathbf{x}_i \right)$. Then for $\hat{A}$, consider its compact SVD $\hat{A} = \hat{U}_1\hat{D}_1\hat{V}_1^\top$ and the full SVD $\hat{A} = \hat{U}\hat{D}\hat{V}^\top$ in which $\hat{U} = [\hat{U}_1 \mid \hat{U}_2]$, $\hat{V} = [\hat{V}_1 \mid \hat{V}_2]$. By Lemma \ref{subdifferntial}, we know that there exists $\hat{E}$ such that
$
\hat{U}_1\hat{V}_1^\top + \hat{U}_2\hat{E}\hat{V}_2^\top = \sum_{i = 1}^n \hat{\alpha}_i y_i \Phi \left(\mathbf{x}_i \right)
$.
Now we construct the kernel generating matrix under the intuition to avoid the basis function matrix $\Phi(\mathbf{x}_i)$. Let $T = \hat{U}_1\hat{V}_1^\top + \hat{U}_2\hat{E}\hat{V}_2^\top = \sum_{i = 1}^n \hat{\alpha}_i y_i \Phi \left(\mathbf{x}_i \right)$. By computing $S = T^\top T$ we get the kernel generating matrix $K(\mathbf{x}_i, \mathbf{x}_j)$ on the right side of the equation. For the left side, we know by the properties of SVD that $\hat{U}_1$ and $\hat{U}_2$ are semi-orthogonal matrices, i.e., $\hat{U}_1^\top\hat{U}_1 = \mathbf{I}_{r}$, $\hat{U}_2^\top\hat{U}_2 = \mathbf{I}_{(d_2-r)}$, and also $\hat{U}_1 \perp \hat{U}_2$, i.e., $\hat{U}_1^\top\hat{U}_2 = \hat{U}_2^\top\hat{U}_1 = \mathbf{0}$. Therefore, by plugging in $T = \hat{U}_1\hat{V}_1^\top + \hat{U}_2\hat{E}\hat{V}_2^\top$ and with some simplification, we get
$
S = \hat{V}_1\hat{V}_1^\top + \hat{V}_2\hat{E}^\top\hat{E}\hat{V}_2^\top
$.
Also, $\hat{V}_1 \perp \hat{V}_2$, then $\hat{V}_1\perp\hat{V}_2\hat{E}^\top$. Consequently, there exists $\tilde{V} \in \R^{p \times p}$ for the eigendecompositions $S = \tilde{V} \Lambda \tilde{V}^{-1}$, $\hat{V}_1\hat{V}_1^\top = \tilde{V} \Lambda_v \tilde{V}^{-1}$ and $\hat{V}_2\hat{E}^\top\hat{E}\hat{V}_2^\top = \tilde{V} \Lambda_e \tilde{V}^{-1}$, such that
$
\Lambda = \Lambda_v + \Lambda_e
$.
Since $\hat{V}_1$ is semi-orthogonal, there are only $0$'s and $1$'s on the diagonal of $\Lambda_v$. Meanwhile, by the Lemma \ref{subdifferntial}, we know that $\sigma_{max}\left(E \right) \leq 1$. That is to say, the elements on the diagonal of $\Lambda_e$ lie in the range $[0,1]$. Moreover, given that $\hat{V}_1 \perp \hat{V}_2\hat{E}^\top$, we know that
$\Lambda_{v,ii} = 0$ for all $i$ with $\Lambda_{e,ii} > 0$, and $1$ otherwise.
Even though the eigenvalue of $\hat{V}_2\hat{E}^\top\hat{E}\hat{V}_2^\top$ could also be $1$, but unless constructed, the probability of $\sigma_{max}(\hat{E})$ being  exactly equal to 1 is extremely close to 0 for any distribution with full support, since the condition has Lebesgue measure zero. Therefore, the eigenvectors $\big[\tilde{V}(i)\big]$ with $\Lambda_{ii} = 1$ are the eigenvectors of $\hat{V}_1\hat{V}_1^\top$. Since $\hat{V}_1$ is from the compact SVD with $\rank(\hat{V}_1) = r$, $\hat{V}_1 = \big[\tilde{V}(i)\big]$ is the recovered linear weight. \hfill $\blacksquare$

\begin{remark}
One of the benefits we enjoy from our recovery algorithm is that we do not need to set the number of filters as a hyperparameter, but can directly deduce it from the linear weight we recover. That is, we do not fix the number of convolutional filters $r$, but deduce $r$ by letting $r := \rank( \hat{A})$, or equivalently, $r := \rank ( \hat{V}_1 )$, i.e., $r$ is the number of $1$'s in the diagonal of $\Lambda$. Furthermore, the imposed nuclear norm constraint on $A$ encourages a low rank which implicitly reduces the parameter size in attempt to find the minimum number of filters necessary.
\end{remark}

\subsubsection{Recovering the convolutional weight}
In the case of the convolutional weight, there is no closed form expression for $\hat{U}_1 \in \mathbb{R}^{d_2 \times r}$ because we cannot exactly compute $\Phi(\mathbf{x}_i) \in \mathbb{R}^{d_2 \times p}$ and $d_2$ could be infinity. We provide an approach to compute the output of the convolutional weight directly, which implicitly contains $\hat{U}_1$, and can be used for the purpose of layerwise training as well as making new predictions. 
\begin{theorem} \label{theo_binary_conv}
Given the optimal dual solution $\hat{\alpha}_i$, $i \in [n]$ to the problem in Eq. \eqref{dual}, the convolutional weight $\hat{U}_1$ can be implicitly recovered by the convolution output, that is, $\forall i \in [n]$,
\begin{align}
\Phi\left( \mathbf{x}_i \right)^\top\hat{U}_1 = \sum_{j = 1}^n \hat{\alpha}_j y_j K(\mathbf{x}_i,\mathbf{x}_j) \hat{V}_1. \label{conv_output}
\end{align}
\end{theorem}
\textit{Proof Sketch:} By Lemma \ref{stationary_lemma} and Lemma \ref{subdifferntial} we know that $\hat{V}_1\hat{U}_1^\top =  \sum_{j = 1}^n \hat{\alpha}_j y_j \Phi \left(\mathbf{x}_j \right)^\top - \hat{V}_2\hat{E}^\top\hat{U}_2^\top$. For both sides of the equation, multiply with $\hat{V}_1^\top$ on the left and $\Phi(\mathbf{x}_i)$ on the right, we have
$
\hat{V}_1^\top\hat{V}_1\hat{U}_1^\top \Phi\left( \mathbf{x}_i \right) =  \sum_{j = 1}^n \hat{\alpha}_j y_j \hat{V}_1^\top\Phi \left(\mathbf{x}_j \right)^\top \Phi\left( \mathbf{x}_i \right) - \hat{V}_1^\top\hat{V}_2\hat{E}^\top\hat{U}_2^\top \Phi\left( \mathbf{x}_i \right)$.
Note that $\hat{V}_1$ is semi-orthogonal and $\hat{V}_1 \perp \hat{V}_2$, thus $\hat{V}_1^\top\hat{V}_1 = \mathbf{I}$ and $\hat{V}_1^\top\hat{V}_2 = \mathbf{0}$. This concludes the proof. \hfill $\blacksquare$

After recovering the convolution output, one can train a multi-layer DCCNN by the layerwise training approach applied in \cite{zhang2017convexified, pmlr-v97-belilovsky19a}. That is to regard the vectorized output of the convolution weight $\vecrize\big(\Phi( \mathbf{x}_i )^\top\hat{U}_1\big)$ as the input for the next layer. Note that the vectorization automatically takes care of the case for multiple-channel input or output, as it will eventually be transformed into a vector. The complete algorithm of learning a $\mathcal{D}$-layer DCCNN is illustrated in Algorithm \ref{train_alg} in Appendix \ref{training_algorithm}.

Our way of recovering the weight output can also be adapted to other CNN techniques such as \emph{Average Pooling}. To achieve average pooling, one can generate a pooling matrix $G$ and simply multiply it with the convolution output before feeding it into the next layerwise training. One can refer to Appendix \ref{avgpooling} for the generation of average pooling matrix.

Making prediction with a new sample $\mathbf{x}_{new}$ can also be achieved with the dual solution and the kernel generating matrix. We can make predictions with the two-layer DCCNN by firstly plugging the sample into Eq. \eqref{conv_output}, then multiplying the convolutional output with the linear weight, and finally computing $f\left( \mathbf{x}_{new}\right) = \sign \big( \Tr ( \sum_{j = 1}^n \hat{\alpha}_j y_j K(\mathbf{x}_{new}, \mathbf{x}_j) \hat{L}\hat{L}^\top ) \big)$. The algorithm for making predictions with a layerwise trained $\mathcal{D}$-layer DCCNN is described in Algorithm \ref{pred_alg} in Appendix \ref{predict_algorithm}.

\subsection{Extension to Multiclass Classification}
Our proposed method can be easily extended to multi-class classification. In this section, we provide the results for the multi-class version and leave the derivations and proofs in Appendix \ref{multi_proof}, as they closely resemble the derivations for binary classification. 

We start from the multiclass version of loss minimization in Eq. \eqref{multi_hinge_loss}. To be consistent with the binary classification formulation, for the parameters $\{A_k\}_{k=1}^m$,  we define $A = \left[A_1, ... , A_m \right] \in \R^{d_2 \times mp}$. Similarly, for $\Phi\left( \mathbf{x}_i\right) \in \R^{d_2 \times p}$, we define $\Phi_k'\left( \mathbf{x}_i\right) = \big[\mathbf{0}_{d_2 \times (k-1)p},\Phi\left( \mathbf{x}_i\right), \mathbf{0}_{d_2 \times (m-k)p} \big] \in \R^{d_2 \times mp}$. In this way we have $\Tr \big( \Phi (\mathbf{x}_i)^\top A_k \big) = \Tr \big( \Phi'_k (\mathbf{x}_i)^\top A \big)$, and the minimization in Eq. \eqref{multi_hinge_loss} is equivalent to the following optimization problem: 
\begin{align} 
\minimize_{A} 
\quad & 
\normnu{A} 
+ c \sum_{k \neq y_i} \sum_{i = 1}^n \ell(\xi_{k,i}) \label{multi_real_primal} \\
\st \quad & \Tr \big( \Phi'_{y_i} (\mathbf{x}_i)^\top A \big) - \Tr \big( \Phi'_k(\mathbf{x}_i)^\top A \big) \nonumber \\
& \geq \xi_{k,i} \,,  
\quad \forall i \in [n], \ k \neq y_i \,  \nonumber
\end{align} 
in which $\ell(\cdot)$ is convex and non-increasing, $\xi_{k,i}$ is an equivalent variable of $\Tr \big( \Phi'_{y_i} (\mathbf{x}_i)^\top A \big) - \Tr \big( \Phi'_k(\mathbf{x}_i)^\top A \big)$, and $c > 0$ is some hyperparameter.

We can derive the dual of the optimization problem in Eq. \eqref{multi_real_primal} as in Section \ref{dual_derivation}. With an adaptation of the kernel generating matrix, we give out the dual problem for multi-class classification.
\begin{theorem} \label{multi_dual_theorem}
For all dual variables $\alpha_{k,i} \geq 0$ with $i \in [n]$ and $k \in [m]$, the dual problem of Eq. \eqref{multi_real_primal} is given by:
\begin{align} 
\maximize_{\alpha} 
\quad & 
- c\sum_{i = 1}^n \sum_{k = 1}^m \ell^\ast \left(- \frac{\alpha_{k,i}}{c} \right)   \label{multi_real_dual} \\
\st \quad & 
\lambda_{max} \Big(  \sum_{i=1}^n\sum_{j=1}^n\sum_{k=1}^m\alpha'_{k,i} \alpha'_{k,j} K'_k (\mathbf{x}_i, \mathbf{x}_j) \Big) \leq 1 \,,  
 \nonumber \\
& \alpha_{k,i} \geq 0, \ \alpha_{y_i,i} = 0, \ \forall i \in [n], \ \forall k \in [m] \,, \nonumber 
\end{align} 
in which $\alpha'_{k,i} = \sum_{s = 1}^m \alpha_{s,i} \mathbbm{1}\left[ k = y_i \right] - \alpha_{k,i}$,  $\ell^\ast(\cdot)$ is the Fenchel conjugate of the loss function $\ell(\cdot)$, and the kernel generating matrix is constructed in $K'_k(\mathbf{x}_i, \mathbf{x}_j) = \diag\big( \mathbf{0}_{p \times p}, ..., \underbrace{K(\mathbf{x}_i, \mathbf{x}_j)}_{ k^{th} \ block \ diagonal}, ... , \mathbf{0}_{p \times p}\big)$, with $K (\mathbf{x}_i, \mathbf{x}_j) = \Phi (\mathbf{x}_i)^\top \Phi (\mathbf{x}_j)$.
\end{theorem}

After solving the dual problem, we can recover the linear weight and the convolution output by Theorem \ref{theom_multi_linear} and Theorem \ref{theom_multi_conv} using the dual solution $\{\{\hat{\alpha}_{k,i}\}_{k=1}^m\}_{i=1}^n$ and the kernel generating matrix.
\begin{theorem} \label{theom_multi_linear}
We can recover the linear weight $\hat{V}_1 \in \R^{mp \times r}$ for multi-class classification by 
$$\hat{V}_1 = \big[\tilde{V}(i)\big], \ for \ all \ i \ with \ \lambda_i = 1.$$ Here $\tilde{V} \in \R^{mp \times mp}$ comes from $$\tilde{V} \Lambda \tilde{V}^{-1} = \sum_{i=1}^n\sum_{j=1}^n\sum_{k=1}^m\alpha'_{k,i} \alpha'_{k,j}  K'_k \left(\mathbf{x}_i, \mathbf{x}_j\right)$$ where $\Lambda = \diag(\lambda_1, ..., \lambda_{mp}) \in \R^{mp \times mp}$, and $\alpha'_{k,i} = \sum_{s = 1}^m \alpha_{s,i} \mathbbm{1}\left[ k = y_i \right] - \alpha_{k,i}$.
\end{theorem}
\begin{theorem} \label{theom_multi_conv}
We can implicitly recover the convolutional weight in the convolution output for multi-class classification. That is, for all $i \in [n]$, we have $$\Phi\left( \mathbf{x}_i \right)^\top\hat{U}_1 = \sum_{j = 1}^n \sum_{k=1}^m \alpha'_{k,j} \big [ \mathbf{0}_{p \times (k-1)p}, K \left(\mathbf{x}_i, \mathbf{x}_j\right), \mathbf{0}_{p \times (m-k)p} \big] \hat{V}_1$$.
\end{theorem}
Same as binary classification, $\vecrize \big( \Phi\left( \mathbf{x}_i \right)^\top\hat{U}_1 \big)$ can be regarded as the input for the next layerwise training. To make a new prediction for sample $\mathbf{x}_{new}$ with the multiclass model, for every $k \in [m]$, we can extract the linear weight of each class $\hat{L}_{k} \in \R^{p \times r}$ from $\hat{L} = [\hat{L}_{1}^\top, ..., \hat{L}_{m}^\top]^\top \in \R^{mp \times r}$, and then make predictions with the multi-class predictor $$f( \mathbf{x}_{new} ) = \argmax_{k \in [m]} \big \{\Tr \big( \sum_{i = 1}^n \sum_{l=1}^m \alpha'_{l,i} \big [ \mathbf{0}_{p \times (l-1)p}, K (\mathbf{x}_{new}, \mathbf{x}_i), \mathbf{0}_{p \times (m-l)p} \big] \hat{L}\hat{L}_{k}^\top \big) \big \}$$.

\begin{table*}[htp]
 \begin{subtable}[h]{\textwidth}
  \centering
  \begin{tabular}{llllllll}
    \toprule
    \multirow{2}{*}{Dataset}  & \multirow{2}{*}{Architecture} & \multicolumn{4}{c}{CCNN} &  \multirow{2}{*}{SGD} & \multirow{2}{*}{DCCNN} \\
    \cmidrule(lr){3-6}
    & & $UD^{\frac{1}{2}}$ & $UD^{\frac{1}{2}}V$ & $K^{\frac{1}{2}}$ & Cholesky &  &  \\
    \midrule
    \multirow{2}{*}{MNIST}  & 1-Layer & 90.3\% & 90.3\% & 88.8\% & 93.7\% & 90.2\% & 94.8\%\\
    \cmidrule(lr){3-8}
                            & 2-Layer & 93.7\% & 93.5\% & 93.5\% & 96.7\% & 95.2\% & 96.0\% \\
     \bottomrule
     \toprule
    \multirow{2}{*}{ImageNet} & AlexNet & 62.3\% & 62.3\% & 53.2\% & 62.0\% & 87.1\% & 83.3\%  \\
    \cmidrule(lr){3-8}
            & VGG11 & 55.3\% & 52.5\% & 59.8\% & 57.0\% & 89.2\% & 85.0\% \\
    \bottomrule
  \end{tabular}
  \vspace{5pt}
  \caption{Results on MNIST and ImageNet Binary Classification. }
  \label{binary_res}
  \vspace{0pt}
 \end{subtable}

 \vspace{10pt}
 
 \begin{subtable}[h]{\textwidth}
  \centering
  \begin{tabular}{llllllll}
    \toprule
    \multirow{2}{*}{Dataset} & \multirow{2}{*}{Architecture} & \multicolumn{4}{c}{CCNN} &  \multirow{2}{*}{SGD} & \multirow{2}{*}{DCCNN} \\
    \cmidrule(lr){3-6}
    & & $UD^{\frac{1}{2}}$ & $UD^{\frac{1}{2}}V$ & $K^{\frac{1}{2}}$ & Cholesky &  &  \\
    \midrule
    {MNIST}  &  1-Layer & 75.4\% & 75.4\% & 82.3\% & 85.9\% & 87.0\% & 85.3\%\\
     \bottomrule
  \end{tabular}
  \vspace{5pt}
  \caption{Results on MNIST Multiclass Classification.}
  \label{multiexp_res}
  \vspace{0pt}
 \end{subtable}
 \caption{Experiment results for binary and multiclass classification. Methods compared include CCNN, CNN trained with SGD, and our proposed DCCNN. For CCNN we include 4 different ways for kernel matrix factorization: (1) $Q = UD^{\frac{1}{2}}$ for $K = UDU^\top$; (2) $Q = UD^{\frac{1}{2}}V$ for $K = UDU^\top$, and random orthonormal matrix $V$; (3) $Q = K^{\frac{1}{2}}$; (4) $Q$ from the Cholesky Decomposition of $K$. The X-layer in the architecture column refers to the number of convolutional layers, as only one linear layer is concatenated at the end for classification.}
 \label{exp_res}
\end{table*}

%% file: 4_experiment.tex
\section{Experiments} \label{exp_section}
In this section we evaluate DCCNN on real-world data as a sanity check for the proposed method. The performance is by no means state-of-the-art for relevant tasks.
We apply hinge loss for evaluation. The specific form of the problem in Eq. \eqref{real_dual} with hinge loss can be found in Appendix \ref{hinge_optimization}.
We briefly report results in Table \ref{exp_res}, and leave the detailed analysis in Appendix \ref{experiment}. 
In Table \ref{binary_res}, for the binary classification task of the MNIST data \cite{726791}, we can see the DCCNN outperforms CNNs optimized by SGD and different kernel matrix factorizations for CCNN on one-conv-layer and two-conv-layer networks with only one exception, which verifies the effectiveness of DCCNN. On the more complicated ImageNet dataset \cite{5206848}, DCCNN also performs comparably well with the end-to-end SGD optimized CNNs under both AlexNet \cite{krizhevsky2012imagenet} and VGG11 \cite{simonyan2014very} architectures, and significantly outperforms the CCNN method. From Table \ref{multiexp_res}, we see that in multiclass classification, the performance level of DCCNN is better than various factorization versions of CCNN and is comparable with SGD.

\section{Conclusion}
Dual convexified convolutional neural networks (DCCNNs) constructed the dual training program of a primal learning problem motivated by convexified convolutional neural networks (CCNNs) through careful analysis of the Fenchel conjugates and Karush-Kuhn-Tucker (KKT) conditions. DCCNN eliminates the theoretical ambiguity and reduces the computational overhead of constructing and factorizing a large kernel matrix in CCNN. To tackle the challenge of non-differentiable nuclear norm and no closed-form expression to recover the primal solution, a highly novel weight recovery algorithm is proposed to recover the linear weight and the convolution output directly, without computing convolutional weight or approximating basis function matrix.

%% file: 5_supplement.tex
\newpage
\onecolumn

\setcounter{section}{0}
\section{Detailed Proofs for Binary Classification} \label{bi_proof}
\subsection{Proof of Theorem \ref{dual_theorem}}
\renewcommand{\thetheorem}{\ref{dual_theorem}}
\begin{theorem}
The dual problem of Eq. \eqref{primal} is given by: 
\begin{align} 
\maximize_{\alpha} 
\quad & 
-c\sum_{i = 1}^n \ell^\ast\left(-\frac{\alpha_i}{c}\right) \nonumber \\
\st \quad & 
\lambda_{max} \Big(  \sum_{i = 1}^n \sum_{j = 1}^n \alpha_i \alpha_j y_i y_j K(\mathbf{x}_i, \mathbf{x}_j) \Big) \leq 1 \,, \tag{\ref{real_dual}}  \\ 
& \alpha_i \geq 0 \,, 
\quad \forall i \in [n] \,, \nonumber  
\end{align} 
in which $\alpha_i$'s are the dual variables, $\ell^\ast(\cdot)$ is the Fenchel conjugate of the loss function $\ell(\cdot)$, $K\left(\mathbf{x}_i, \mathbf{x}_j\right) = \Phi \left(\mathbf{x}_i\right)^\top \Phi \left(\mathbf{x}_j\right)$ is a kernel generating matrix, such that its entries are composed of the kernel function $\mathcal{K} \left( \cdot, \cdot \right) = \phi \left( \cdot \right)^\top\phi \left( \cdot \right)$ taking one patch from each of the two samples as input:  
\begin{align}
K\left(\mathbf{x}_i, \mathbf{x}_j\right) = 
\begin{pmatrix}
\mathcal{K}\big(z_1(\mathbf{x}_i), z_1(\mathbf{x}_j)\big) & \mathcal{K}\big(z_1(\mathbf{x}_i), z_2(\mathbf{x}_j)\big) & \cdots & \mathcal{K}\big(z_1(\mathbf{x}_i), z_p(\mathbf{x}_j)\big) \\
\mathcal{K}\big(z_2(\mathbf{x}_i), z_1(\mathbf{x}_j)\big) & \mathcal{K}\big(z_2(\mathbf{x}_i), z_2(\mathbf{x}_j)\big) &\cdots & \mathcal{K}\big(z_2(\mathbf{x}_i), z_p(\mathbf{x}_j)\big) \\
\vdots  & \vdots  & \ddots & \vdots \\
\mathcal{K}\big(z_p(\mathbf{x}_i), z_1(\mathbf{x}_j)\big) & \mathcal{K}\big(z_p(\mathbf{x}_i), z_2(\mathbf{x}_j)\big) &\cdots & \mathcal{K}\big(z_p(\mathbf{x}_i), z_p(\mathbf{x}_j)\big)
\end{pmatrix}. \tag{\ref{kernel_matrix}} 
\end{align}
\end{theorem}

\begin{proof}
For Lagrangian variables $\alpha_i \geq 0$ and $\beta_i \geq 0$, we can write the Lagrangian function of the optimization problem in Eq. \eqref{primal} as 
\begin{align} \nonumber
\mathcal{L}\left(A, \xi, \alpha\right) &= \left \Vert A \right \Vert_\ast +  c \sum_{i = 1}^n \ell(\xi_i) + \sum_{i = 1}^n \alpha_i \left( \xi_i - y_i\Tr \left( \Phi \left(\mathbf{x}_i\right)^\top A \right) \right) \\ \nonumber
&= \left \Vert A \right \Vert_\ast - \sum_{i = 1}^n \alpha_i y_i\Tr \left( \Phi \left(\mathbf{x}_i\right)^\top A \right) + c \sum_{i = 1}^n \ell(\xi_i) + \sum_{i = 1}^n \alpha_i \xi_i \\ \nonumber
&= \left \Vert A \right \Vert_\ast - \Tr \left( \sum_{i = 1}^n \alpha_i y_i\Phi \left(\mathbf{x}_i\right)^\top A \right) + c \sum_{i = 1}^n \ell(\xi_i) + \sum_{i = 1}^n \alpha_i \xi_i \\
&= \left \Vert A \right \Vert_\ast - \left \langle \sum_{i = 1}^n \alpha_i y_i \Phi \left(\mathbf{x}_i\right), A \right \rangle + c \sum_{i = 1}^n \ell(\xi_i) + \sum_{i = 1}^n \alpha_i \xi_i,  \tag{\ref{Lagrangian}}
\end{align} 
in which Eq. \eqref{Lagrangian} utilizes the trace definition of the matrix inner product.

By the definition of the dual function, 
\begin{align}
g\left(\alpha\right) &= \min_{A, \xi} \mathcal{L}\left(A, \xi, \alpha \right) \nonumber \\
&= \min_{A} \left \Vert A \right \Vert_\ast - \left \langle \sum_{i = 1}^n \alpha_i y_i \Phi \left(\mathbf{x}_i\right), A \right \rangle + \min_{\xi}  c \sum_{i = 1}^n \ell(\xi_i) + \sum_{i = 1}^n \alpha_i \xi_i \nonumber \\
&= -\underbrace{\max_{A} \left\{ \left \langle \sum_{i = 1}^n \alpha_i y_i \Phi \left(\mathbf{x}_i\right), A \right \rangle - \left \Vert A \right \Vert_\ast \right\}}_{g_1} - \underbrace{\max_{\xi} \left\{ - c \sum_{i = 1}^n \ell(\xi_i) - \sum_{i = 1}^n \alpha_i \xi_i \right\}}_{g_2}
\end{align} 
We notice that by the definition of conjugate function \cite{boyd2004convex}, $g_1$ takes the form of the conjugate function of the nuclear norm, i.e. for $f_0 = \left \Vert A \right \Vert_\ast$
\begin{align}
g_1 = f_0^\ast \left( \sum_{i = 1}^n \alpha_i y_i \Phi \left(\mathbf{x}_i\right) \right), \tag{\ref{conjugate}}
\end{align}
in which 
\begin{align} 
f_0^\ast \left( \cdot \right) = 
\begin{cases}
\ 0 \ , & \left \Vert \cdot \right \Vert_2 \leq 1 \nonumber \\
\infty, & otherwise
\end{cases}
\end{align}
Also, 
\begin{align}
    g_2 &= \max_{\xi}\left\{ - c \sum_{i = 1}^n \ell(\xi_i) - \sum_{i = 1}^n \alpha_i \xi_i \right\}\nonumber \\
        &= c\sum_{i=1}^n \max_{\xi_i} \left(-\frac{\alpha_i}{c} \cdot \xi_i - \ell(\xi_i) \right) \nonumber \\
        &= c\sum_{i=1}^n \ell^\ast\left(- \frac{\alpha_i}{c}\right), \nonumber
\end{align} 
in which $\ell^\ast(\cdot) = \sup_{\xi_i} \left\{ \langle \cdot, \xi_i \rangle - \ell(\xi_i) \right\}$ is the conjugate function of the loss function $\ell(\xi_i)$.
Therefore, the dual problem of Eq. \eqref{primal} is as follows:
\begin{align} 
\maximize_{\alpha} 
\quad & 
-c\sum_{i = 1}^n \ell^\ast\left(-\frac{\alpha_i}{c}\right) \nonumber \\
\st \quad & 
\big \Vert \sum_{i = 1}^n \alpha_i y_i \Phi \left(\mathbf{x}_i\right) \big \Vert_2 \leq 1 \,,  \tag{\ref{dual}} \\
& \alpha_i \geq 0 \,, 
\quad \forall i \in [n] \,. \nonumber
\end{align}
The spectral norm constraint in Eq. \eqref{dual} is equivalent to
\begin{align}
\left \Vert \sum_{i = 1}^n \alpha_i y_i \Phi \left(\mathbf{x}_i\right) \right \Vert_2^2 \leq 1. \nonumber
\end{align}
Furthermore,
\begin{align}
\left \Vert \sum_{i = 1}^n \alpha_i y_i \Phi \left(\mathbf{x}_i\right) \right \Vert_2^2 &= \sigma_{max} \left(  \sum_{i = 1}^n \alpha_i y_i \Phi \left(\mathbf{x}_i\right) \right)^2 \nonumber \\
&= \lambda_{max} \left( \left(  \sum_{i = 1}^n \alpha_i y_i \Phi \left(\mathbf{x}_i\right) \right)^\top \left(  \sum_{j = 1}^n \alpha_j y_j \Phi \left(\mathbf{x}_j\right) \right) \right) \nonumber \\
&= \lambda_{max} \left(  \sum_{i = 1}^n \sum_{j = 1}^n \alpha_i \alpha_j y_i y_j \Phi \left(\mathbf{x}_i\right)^\top \Phi \left(\mathbf{x}_j\right) \right).
\end{align}
Therefore we have:
\begin{align} 
\maximize_{\alpha} 
\quad & 
-c\sum_{i = 1}^n \ell^\ast\left(-\frac{\alpha_i}{c}\right) \nonumber \\
\st \quad & 
\lambda_{max} \Big(  \sum_{i = 1}^n \sum_{j = 1}^n \alpha_i \alpha_j y_i y_j K(\mathbf{x}_i, \mathbf{x}_j) \Big) \leq 1 \,, \tag{\ref{real_dual}}  \\ 
& \alpha_i \geq 0 \,, 
\quad \forall i \in [n] \,, \nonumber  
\end{align}
in which $K\left(\mathbf{x}_i, \mathbf{x}_j\right) = \Phi \left(\mathbf{x}_i\right)^\top \Phi \left(\mathbf{x}_j\right)$ is the kernel generating matrix, and
\begin{align}
K\left(\mathbf{x}_i, \mathbf{x}_j\right) = 
\begin{pmatrix}
\mathcal{K}\left(z_1\left(\mathbf{x}_i\right), z_1\left(\mathbf{x}_j\right)\right) & \mathcal{K}\left(z_1\left(\mathbf{x}_i\right), z_2\left(\mathbf{x}_j\right)\right) & \cdots & \mathcal{K}\left(z_1\left(\mathbf{x}_i\right), z_p\left(\mathbf{x}_j\right)\right) \\
\mathcal{K}\left(z_2\left(\mathbf{x}_i\right), z_1\left(\mathbf{x}_j\right)\right) & \mathcal{K}\left(z_2\left(\mathbf{x}_i\right), z_2\left(\mathbf{x}_j\right)\right) &\cdots & \mathcal{K}\left(z_2\left(\mathbf{x}_i\right), z_p\left(\mathbf{x}_j\right)\right) \\
\vdots  & \vdots  & \ddots & \vdots \\
\mathcal{K}\left(z_p\left(\mathbf{x}_i\right), z_1\left(\mathbf{x}_j\right)\right) & \mathcal{K}\left(z_p\left(\mathbf{x}_i\right), z_2\left(\mathbf{x}_j\right)\right) &\cdots & \mathcal{K}\left(z_p\left(\mathbf{x}_i\right), z_p\left(\mathbf{x}_j\right)\right)
\end{pmatrix}, \tag{\ref{kernel_matrix}}
\end{align}
in which $\mathcal{K} \left( \cdot, \cdot \right) = \phi \left( \cdot \right)^\top\phi \left( \cdot \right)$ is the kernel function.
\end{proof}

\subsection{Proof of Theorem \ref{theo_binary_linear}}
\renewcommand{\thetheorem}{\ref{theo_binary_linear}}
\begin{theorem}
Given the optimal dual solution $\{\hat{\alpha}_i\}_{i=1}^n$ to the problem in Eq. \eqref{real_dual}, let $\tilde{V} \in \R^{p \times p}$ be the matrix of eigenvectors from the eigendecomposition 
\begin{align}
\tilde{V} \Lambda \tilde{V}^{-1} = \sum_{i = 1}^n \sum_{j = 1}^n \hat{\alpha}_i \hat{\alpha}_j y_i y_j K(\mathbf{x}_i,\mathbf{x}_j),
\end{align}
where $\Lambda = \diag(\lambda_1, ..., \lambda_p) \in \R^{p \times p}$. The linear weight $\hat{V}_1$ can be recovered by the eigenvectors in $\tilde{V}$ corresponding to the eigenvalue of 1, that is,
\begin{align}
\hat{V}_1 = \left[\tilde{V}(i)\right] \in \R^{p \times r}, \text{ for all } i \text{ such that } \lambda_i = 1, \tag{\ref{linear}}
\end{align}
in which $\tilde{V}(i)$ denotes the $i^{th}$ column of $\tilde{V}$ and $[\cdot]$ denotes the concatenation of these columns.
\end{theorem}
\begin{proof}
By the stationary condition in Lemma \ref{stationary_lemma}, we know that the there exists one element in the subdifferential of $\Vert A \Vert$ at value $\hat{A}$ that satisfies 
$$\partial \Vert \hat{A} \Vert_\ast = \sum_{i = 1}^n \hat{\alpha}_i y_i \Phi \left(\mathbf{x}_i \right).$$ Then for $\hat{A}$, consider its compact SVD $\hat{A} = \hat{U}_1\hat{D}_1\hat{V}_1^\top$ and the full SVD $\hat{A} = \hat{U}\hat{D}\hat{V}^\top$ in which $\hat{U} = [\hat{U}_1 \mid \hat{U}_2]$, $\hat{V} = [\hat{V}_1 \mid \hat{V}_2]$. By Lemma \ref{subdifferntial}, we know that $\exists \   \hat{E}$ such that 
\begin{align}
\hat{U}_1\hat{V}_1^\top + \hat{U}_2\hat{E}\hat{V}_2^\top = \sum_{i = 1}^n \hat{\alpha}_i y_i \Phi \left(\mathbf{x}_i \right) \label{uv}
\end{align}
Now we construct the kernel generating matrix with Eq. \eqref{uv} under the intuition to avoid the basis function matrix $\Phi(\mathbf{x}_i)$. Let $T = \hat{U}_1\hat{V}_1^\top + \hat{U}_2\hat{E}\hat{V}_2^\top = \sum_{i = 1}^n \hat{\alpha}_i y_i \Phi \left(\mathbf{x}_i \right)$ and compute $S = T^\top T$. For the right side of Eq. \eqref{uv}, we get the expression containing the kernel generating matrix $K(\mathbf{x}_i, \mathbf{x}_j)$:
\begin{align}
S &= \sum_{i = 1}^n \sum_{j = 1}^n \hat{\alpha}_i \hat{\alpha}_j y_i y_j \Phi \left(\mathbf{x}_i\right)^\top \Phi \left(\mathbf{x}_j\right) \nonumber \\
&= \sum_{i = 1}^n \sum_{j = 1}^n \hat{\alpha}_i \hat{\alpha}_j y_i y_j K(\mathbf{x}_i,\mathbf{x}_j).
\end{align}
For the left side of Eq. \eqref{uv}, we know by the properties of SVD that $\hat{U}_1$ and $\hat{U}_2$ are semi-orthogonal matrices, i.e., $\hat{U}_1^\top\hat{U}_1 = \mathbf{I}_{r}$, $\hat{U}_2^\top\hat{U}_2 = \mathbf{I}_{(d_2-r)}$, and also $\hat{U}_1 \perp \hat{U}_2$, i.e., $\hat{U}_1^\top\hat{U}_2 = \hat{U}_2^\top\hat{U}_1 = \mathbf{0}$. Therefore, by plugging in $T = \hat{U}_1\hat{V}_1^\top + \hat{U}_2\hat{E}\hat{V}_2^\top$, we get
\begin{align}
S &= \left ( \hat{U}_1\hat{V}_1^\top+\hat{U}_2\hat{E}\hat{V}_2^\top \right)^\top \left ( \hat{U}_1\hat{V}_1^\top+\hat{U}_2\hat{E}\hat{V}_2^\top \right) \nonumber \\
&= \hat{V}_1 \hat{U}_1^\top \hat{U}_1\hat{V}_1^\top + \hat{V}_1 \hat{U}_1^\top\hat{U}_2\hat{E}\hat{V}_2^\top + \hat{V}_2\hat{E}^\top\hat{U}_2^\top\hat{U}_1\hat{V}_1^\top + \hat{V}_2\hat{E}^\top\hat{U}_2^\top\hat{U}_2\hat{E}\hat{V}_2^\top \nonumber \\
&= \hat{V}_1\hat{V}_1^\top + \hat{V}_2\hat{E}^\top\hat{E}\hat{V}_2^\top.
\end{align}
Now we show that the eigenvectors of $\hat{V}_1\hat{V}_1^\top$ can be found in the eigenvectors of $S$ with corresponding eigenvalues being 1. One way to see this is that, since $\hat{V}_1 \perp \hat{V}_2$, we have,
\begin{align}
\big(\hat{V}_1\hat{V}_1^\top\big)\big(\hat{V}_2\hat{E}^\top\hat{E}\hat{V}_2^\top\big) = \big(\hat{V}_2\hat{E}^\top\hat{E}\hat{V}_2^\top\big)\big(\hat{V}_1\hat{V}_1^\top\big) = \mathbf{0}, \nonumber
\end{align}
meaning that $\hat{V}_1\hat{V}_1^\top$ and $\hat{V}_2\hat{E}^\top\hat{E}\hat{V}_2^\top$ commute. Since $\hat{V}_1\hat{V}_1^\top$ and $\hat{V}_2\hat{E}^\top\hat{E}\hat{V}_2^\top$ are both real symmetric matrices that are diagonalizable, they are simultaneously diagonalizable. That is to say, $\exists \ \tilde{V} \in \R^{p \times p}$ from the eigendecomposition of $S$, such that
\begin{align}
S = \tilde{V} \Lambda \tilde{V}^{-1}, \ \hat{V}_1\hat{V}_1^\top = \tilde{V} \Lambda_v \tilde{V}^{-1}, \ and \ \hat{V}_2\hat{E}^\top\hat{E}\hat{V}_2^\top = \tilde{V} \Lambda_e \tilde{V}^{-1}.
\end{align}
On the diagonal of $\Lambda \in \R^{p \times p}$ are the eigenvalues of $S$, i.e. $\Lambda = \diag(\lambda_1, ..., \lambda_p)$, $\Lambda_v$ the eigenvalues of $\hat{V}_1\hat{V}_1^\top$, and $\Lambda_e$ the eigenvalues of $\hat{T}^\top\hat{T}$. We also have
\begin{align}
\Lambda = \Lambda_v + \Lambda_e.
\end{align}
Since $\hat{V}_1$ is semi-orthogonal, there are only $0$'s and $1$'s on the diagonal of $\Lambda_v$. Meanwhile, by the Lemma \ref{subdifferntial}, we know that $\sigma_{max}\left(E \right) \leq 1$. That is to say, the elements on the diagonal of $\Lambda_e$ lie in the range $[0,1]$. Moreover, given that $\hat{V}_1 \perp \hat{V}_2$, we know that
\begin{align}
\Lambda_{v,ii} = 
\begin{cases}
\ 0 \ , & \forall i \ such \ that \ \Lambda_{e,ii} > 0 \\
\ 1 \, & \forall i \ such \ that \ \Lambda_{e,ii} = 0
\end{cases}
\end{align}
Even though the eigenvalue of $\hat{V}_2\hat{E}^\top\hat{E}\hat{V}_2^\top$ could also be $1$, but unless constructed, the probability of $\sigma_{max}(\hat{E})$ being  exactly equal to 1 is extremely close to 0 for any distribution with full support, since the condition has Lebesgue measure zero. Therefore, the eigenvectors $\big[\tilde{V}(i)\big]$ with $\Lambda_{ii} = 1$ are the eigenvectors of $\hat{V}_1\hat{V}_1^\top$. Since $\hat{V}_1$ is from the compact SVD with $\rank(\hat{V}_1) = r$, $\hat{V}_1 = \big[\tilde{V}(i)\big]$ is the linear weight we recover. We can compute $\tilde{V}$ and $\Lambda$ by
\begin{align}
\tilde{V} \Lambda \tilde{V}^{-1} = \sum_{i = 1}^n \sum_{j = 1}^n \hat{\alpha}_i \hat{\alpha}_j y_i y_j K \left(\mathbf{x}_i, \mathbf{x}_j\right). \nonumber
\end{align}
\end{proof}

\subsection{Proof of Theorem \ref{theo_binary_conv}}
\renewcommand{\thetheorem}{\ref{theo_binary_conv}}
\begin{theorem} 
Given the optimal dual solution $\hat{\alpha}_i$, $i \in [n]$ to the problem in Eq. \eqref{dual}, the convolutional weight $\hat{U}_1$ can be implicitly recovered by the convolution output, that is, $\forall i \in [n]$,
\begin{align}
\Phi\left( \mathbf{x}_i \right)^\top\hat{U}_1 = \sum_{j = 1}^n \hat{\alpha}_j y_j K(\mathbf{x}_i,\mathbf{x}_j) \hat{V}_1. \tag{\ref{conv_output}}
\end{align}
\end{theorem}
\begin{proof}
From Eq. \eqref{uv} we know that 
\begin{align}
\hat{V}_1\hat{U}_1^\top =  \sum_{j = 1}^n \hat{\alpha}_j y_j \Phi \left(\mathbf{x}_j \right)^\top - \hat{V}_2\hat{E}^\top\hat{U}_2^\top.
\end{align}
Then to construct the convolution output, we multiply $\Phi\left( \mathbf{x}_i \right)$ on the right for both sides of the equation:
\begin{align}
\hat{V}_1\hat{U}_1^\top \Phi\left( \mathbf{x}_i \right) &=  \sum_{j = 1}^n \hat{\alpha}_j y_j \Phi \left(\mathbf{x}_j \right)^\top \Phi\left( \mathbf{x}_i \right) - \hat{V}_2\hat{E}^\top\hat{U}_2^\top \Phi\left( \mathbf{x}_i \right) \nonumber \\
&= \sum_{j = 1}^n \hat{\alpha}_j y_j K \left(\mathbf{x}_j, \mathbf{x}_i \right) - \hat{V}_2\hat{E}^\top\hat{U}_2^\top \Phi\left( \mathbf{x}_i \right) \nonumber
\end{align}
For the first term, we have generated the kernel generating matrix. Now we need to eliminate the $\hat{V}_1$ on the left side of the equation to get the exact form of the convolution output, and also avoid computing $\Phi\left( \mathbf{x}_i \right)$ on the right side of the equation. Therefore, for both sides of the equation, we multiply $\hat{V}_1^\top$ on the left, then we get
\begin{align}
\hat{V}_1^\top\hat{V}_1\hat{U}_1^\top \Phi\left( \mathbf{x}_i \right) =  \sum_{j = 1}^n \hat{\alpha}_j y_j \hat{V}_1^\top K \left(\mathbf{x}_j, \mathbf{x}_i \right) - \hat{V}_1^\top\hat{V}_2\hat{E}^\top\hat{U}_2^\top \Phi\left( \mathbf{x}_i \right). \nonumber
\end{align}
Knowing that $\hat{V}_1^\top\hat{V}_1 = \mathbf{I}$ and $\hat{V}_1^\top\hat{V}_2 = \mathbf{0}$, we have
\begin{align}
\hat{U}_1^\top \Phi\left( \mathbf{x}_i \right) =  \sum_{j = 1}^n \hat{\alpha}_j y_j \hat{V}_1^\top K \left(\mathbf{x}_j, \mathbf{x}_i \right).
\end{align}
Taking transpose on both sides of the equation completes the proof.
\end{proof}

\section{Detailed Proofs for Multi-class Classification} \label{multi_proof}
\subsection{Proof of Theorem \ref{multi_dual_theorem}}
\renewcommand{\thetheorem}{\ref{multi_dual_theorem}}
\begin{theorem} 
For all dual variables $\alpha_{k,i} \geq 0$ with $i \in [n]$ and $k \in [m]$, the dual problem of Eq. \eqref{multi_real_primal} is given by: 
\begin{align} 
\maximize_{\alpha} 
\qquad & 
- c\sum_{i = 1}^n \sum_{k = 1}^m \ell^\ast \left(- \frac{\alpha_{k,i}}{c} \right) \nonumber  \\
\st \qquad & 
\lambda_{max} \Big(  \sum_{i=1}^n\sum_{j=1}^n\sum_{k=1}^m\alpha'_{k,i} \alpha'_{k,j} K'_k (\mathbf{x}_i, \mathbf{x}_j) \Big) \leq 1 \,,  
  \tag{\ref{multi_real_dual}} \\
& \alpha_{k,i} \geq 0, \ \alpha_{y_i,i} = 0, \quad \forall i \in [n], \ \forall k \in [m] \,, \nonumber 
\end{align} 
in which $\alpha'_{k,i} = \sum_{s = 1}^m \alpha_{s,i} \mathbbm{1}\left[ k = y_i \right] - \alpha_{k,i}$,  $\ell^\ast(\cdot)$ is the Fenchel conjugate of the loss function $\ell(\cdot)$, and the kernel generating matrix is constructed in $K'_k(\mathbf{x}_i, \mathbf{x}_j) = \diag\big( \mathbf{0}_{p \times p}, ..., \underbrace{K(\mathbf{x}_i, \mathbf{x}_j)}_{ k^{th} \ block \ diagonal}, ... , \mathbf{0}_{p \times p}\big)$, with $K (\mathbf{x}_i, \mathbf{x}_j) = \Phi (\mathbf{x}_i)^\top \Phi (\mathbf{x}_j)$.
\end{theorem}
\begin{proof}
For Lagrangian variables $\alpha_{k,i} \geq 0$, $i \in [n]$, $k \in [m] \setminus \{y_i\}$, we can write the Lagrangian function of Eq. \eqref{multi_real_primal} as 
\begin{align} \nonumber
\mathcal{L}\left(A, \xi, \alpha \right) &= \left \Vert A \right \Vert_\ast + c \sum_{i = 1}^n \sum_{k \neq y_i} \ell(\xi_{k,i}) + \sum_{i = 1}^n \sum_{k \neq y_i} \alpha_{k,i} \left( \xi_{k,i} - \Tr \left( \Phi'_{y_i} \left(\mathbf{x}_i\right)^\top A \right) + \Tr \left( \Phi'_k \left(\mathbf{x}_i\right)^\top A \right) \right). \nonumber
\end{align} 
For the simplicity of notation, we define $\alpha_{y_i, i} = 0$, $\xi_{y_i, i} = 0$, and denote $\mathcal{L} = \mathcal{L}\left(A, \xi, \alpha\right)$. Then the Lagrangian function becomes 
\begin{align} 
\mathcal{L} &= \left \Vert A \right \Vert_\ast + \sum_{i = 1}^n \sum_{k = 1}^m \alpha_{k,i} \left( \Tr \left( \Phi'_k \left(\mathbf{x}_i\right)^\top A \right) - \Tr \left( \Phi'_{y_i} \left(\mathbf{x}_i\right)^\top A \right) \right) + c \sum_{i = 1}^n \sum_{k = 1}^m \ell(\xi_{k,i}) + \sum_{i = 1}^n \sum_{k = 1}^m \alpha_{k,i} \xi_{k,i}. \label{multi_lagrangian}
\end{align}
By the definition of the dual function, 
\begin{align}
g\left(\alpha\right) &= \min_{A, \xi} \mathcal{L}\left(A, \xi, \alpha \right) \nonumber \\
&= \underbrace{\min_{A} \big\{\left \Vert A \right \Vert_\ast - \sum_{i = 1}^n \sum_{k = 1}^m \alpha_{k,i} \left( \Tr \left( \Phi'_{y_i} \left(\mathbf{x}_i\right)^\top A \right) - \Tr \left( \Phi'_k \left(\mathbf{x}_i\right)^\top A \right) \right) \big\}}_{g_1} \nonumber \\
& \quad + \underbrace{\min_{\xi} \big\{ c \sum_{i = 1}^n \sum_{k = 1}^m \ell(\xi_{k,i}) + \sum_{i = 1}^n \sum_{k = 1}^m \alpha_{k,i}\xi_{k,i} \big\}}_{g_2}
\end{align}
For $g_1$, we have
\begin{align} \nonumber
g_1 &= \min_{A} \left \Vert A \right \Vert_\ast - \sum_{k = 1}^m \left( \sum_{i = 1}^n  \alpha_{k,i}  \Tr \left( \Phi'_{y_i} \left(\mathbf{x}_i\right)^\top A \right) - \sum_{i = 1}^n  \alpha_{k,i} \Tr \left( \Phi'_k \left(\mathbf{x}_i\right)^\top A \right) \right) \\ \nonumber
&= \min_{A} \left \Vert A \right \Vert_\ast - \sum_{k = 1}^m \left ( \sum_{i = 1}^n  \sum_{s = 1}^m \alpha_{s,i} \mathbbm{1}\left[ k = y_i \right] \Tr \left( \Phi'_k \left(\mathbf{x}_i\right)^\top A \right) - \sum_{i = 1}^n  \alpha_{k,i} \Tr \left( \Phi'_k \left(\mathbf{x}_i\right)^\top A \right) \right) \\ \nonumber
&= \min_{A} \left \Vert A \right \Vert_\ast - \sum_{k = 1}^m \sum_{i = 1}^n \left ( \sum_{s = 1}^m \alpha_{s,i} \mathbbm{1}\left[ k = y_i \right] - \alpha_{k,i}\right ) \Tr \left( \Phi'_k \left(\mathbf{x}_i\right)^\top A \right) \\ \nonumber
&= - \max_{A} \sum_{k = 1}^m  \sum_{i = 1}^n  \left ( \sum_{s = 1}^m \alpha_{s,i} \mathbbm{1}\left[ k = y_i \right] - \alpha_{k,i} \right ) \Tr \left( \Phi'_k \left(\mathbf{x}_i\right)^\top A \right) - \left \Vert A \right \Vert_\ast \\ \nonumber
&= - \max_{A} \left \langle \sum_{k = 1}^m \sum_{i = 1}^n  \left( \sum\nolimits_{s = 1}^m \alpha_{s,i} \mathbbm{1}\left[ k = y_i \right] - \alpha_{k,i} \right) \Phi'_k \left(\mathbf{x}_i\right), A \right \rangle - \left \Vert A \right \Vert_\ast \\
&= - \max_{A}  \left \langle \sum_{k = 1}^m \sum_{i = 1}^n \alpha'_{k,i}   \Phi'_k \left(\mathbf{x}_i\right), A \right \rangle - \left \Vert A \right \Vert_\ast  \label{g1}
\end{align}
in which Eq. \eqref{g1} utilizes the linearity of trace and the definition of matrix inner product, and for a shorthand notation, $\alpha'_{k,i} = \sum_{s = 1}^m \alpha_{s,i} \mathbbm{1}\left[ k = y_i \right] - \alpha_{k,i}$.

We can see that $g_1$ can be re-written as the conjugate function of the nuclear norm, i.e. for $f_0^\ast$ in Eq. \eqref{conjugate},
\begin{align}
g_1 = - f_0^\ast \left( \sum_{i = 1}^n \sum_{k=1}^m \alpha'_{k,i} \Phi'_k \left(\mathbf{x}_i\right) \right).
\end{align}
Furthermore, we must have $\left \Vert \sum_{i = 1}^n \sum_{k=1}^m \alpha'_{k,i} \Phi'_k \left(\mathbf{x}_i\right) \right \Vert_2 \leq 1$ for $g_1$ to be bounded. 
Also,  
\begin{align}
    g_2 &= \min_{\xi} c \sum_{i = 1}^n \sum_{k = 1}^m \ell(\xi_{k,i}) + \sum_{i = 1}^n \sum_{k = 1}^m \alpha_{k,i}\xi_{k,i} \nonumber \\
        &= c \sum_{i = 1}^n \sum_{k = 1}^m \left(\min_{\xi_{k,i}} \ell(\xi_{k,i}) + \frac{\alpha_{k,i}}{c}\xi_{k,i} \right) \nonumber \\
        &= - c \sum_{i = 1}^n \sum_{k = 1}^m \left(\max_{\xi_{k,i}} \langle - \frac{\alpha_{k,i}}{c}, \xi_{k,i} \rangle - \ell(\xi_{k,i}) \right) \nonumber \\
        &= -c \sum_{i = 1}^n \sum_{k = 1}^m \ell^\ast\left(- \frac{\alpha_{k,i}}{c} \right), \nonumber
\end{align}
in which $\ell^\ast(\cdot) = \sup_{\xi_{k,i}} \left\{ \langle \cdot, \xi_{k,i} \rangle - \ell(\xi_{k,i}) \right\}$ is the conjugate function of the loss function $\ell(\xi_{k,i})$. 
Therefore, the dual problem of Eq. \ref{multi_real_primal} is as follows: 
\begin{align} 
\maximize_{\alpha} 
\qquad & 
-c \sum_{i = 1}^n \sum_{k = 1}^m \ell^\ast\left(- \frac{\alpha_{k,i}}{c} \right) \nonumber  \\
\st \qquad & 
\left \Vert \sum_{i = 1}^n \sum_{k=1}^m \alpha'_{k,i} \Phi'_k \left(\mathbf{x}_i\right) \right \Vert_2 \leq 1 \,,  
  \label{multi_dual} \\
& \alpha_{k,i} \geq 0, \ \alpha_{y_i,i} = 0, \quad \forall i \in [n], \ \forall k \in [m] \,, \nonumber 
\end{align} 
The spectral norm constraint in Eq. \eqref{multi_dual} is equivalent to
\begin{align}
\left \Vert \sum_{i = 1}^n \sum_{k=1}^m \alpha'_{k,i} \Phi'_k \left(\mathbf{x}_i\right) \right \Vert_2^2 \leq 1,
\end{align}
which can be further transformed by
\begin{align}
\left \Vert \sum_{i = 1}^n \sum_{k=1}^m \alpha'_{k,i} \Phi'_k \left(\mathbf{x}_i\right) \right \Vert_2^2 \nonumber &= \sigma_{max} \left(  \sum_{i = 1}^n \sum_{k=1}^m \alpha'_{k,i} \Phi'_k \left(\mathbf{x}_i\right) \right)^2 \nonumber \\
& = \lambda_{max} \left( \left( \sum_{i = 1}^n \sum_{k=1}^m \alpha'_{k,i} \Phi'_k \left(\mathbf{x}_i\right) \right)^\top \left(  \sum_{j = 1}^n \sum_{l=1}^m \alpha'_{l,j} \Phi'_l \left(\mathbf{x}_j\right) \right) \right) \nonumber \\
&= \lambda_{max} \left(  \sum_{i,j} \sum_{k,l} \alpha'_{k,i} \alpha'_{l,j} \Phi'_k \left(\mathbf{x}_i\right)^\top \Phi'_l \left(\mathbf{x}_j\right) \right).
\end{align}
Notice that $\forall k, l \in [m]$, 
\begin{align}
&\Phi'_k \left(\mathbf{x}_i\right)^\top \Phi'_l \left(\mathbf{x}_j\right) = 
\begin{cases}
\ \mathbf{0}_{mp \times mp} \ , & k \neq l \\
\ K'_k\left(\mathbf{x}_i, \mathbf{x}_j\right) \, & k = l
\end{cases}, \label{multi_kernel} \\
&K'_k\left(\mathbf{x}_i, \mathbf{x}_j\right) = \diag\left( \mathbf{0}_{p \times p}, ..., \underbrace{K\left(\mathbf{x}_i, \mathbf{x}_j\right)}_{ k^{th} \ block \ diagonal}, ... , \mathbf{0}_{p \times p} \right), \nonumber
\end{align}
in which $K \left(\mathbf{x}_i, \mathbf{x}_j\right) = \Phi \left(\mathbf{x}_i\right)^\top \Phi \left(\mathbf{x}_j\right)$.
Therefore we have the following dual optimization problem for multiclass classification:
\begin{align} 
\maximize_{\alpha} 
\qquad & 
- c\sum_{i = 1}^n \sum_{k = 1}^m \ell^\ast \left(- \frac{\alpha_{k,i}}{c} \right) \nonumber  \\
\st \qquad & 
\lambda_{max} \Big(  \sum_{i=1}^n\sum_{j=1}^n\sum_{k=1}^m\alpha'_{k,i} \alpha'_{k,j} K'_k (\mathbf{x}_i, \mathbf{x}_j) \Big) \leq 1 \,,  
  \tag{\ref{multi_real_dual}} \\
& \alpha_{k,i} \geq 0, \ \alpha_{y_i,i} = 0, \quad \forall i \in [n], \ \forall k \in [m] \,. \nonumber 
\end{align} 
\end{proof}

\subsection{Proof of Theorem \ref{theom_multi_linear}}
\renewcommand{\thetheorem}{\ref{theom_multi_linear}}
\begin{theorem} 
We can recover the linear weight $\hat{V}_1 \in \R^{mp \times r}$ by
\begin{align}
\hat{V}_1 = \left[\tilde{V}(i)\right], \forall \ i \ such \ that \ \lambda_i = 1, \label{multi_linear}
\end{align}
in which $\tilde{V} \in \R^{mp \times mp}$ comes from
\begin{align}
\tilde{V} \Lambda \tilde{V}^{-1} = \sum_{i=1}^n\sum_{j=1}^n\sum_{k=1}^m\alpha'_{k,i} \alpha'_{k,j} K'_k \left(\mathbf{x}_i, \mathbf{x}_j\right),
\end{align}
and $\Lambda = \diag(\lambda_1, ..., \lambda_{mp}) \in \R^{mp \times mp}$, $\hat{\alpha}'_{k,i} = \sum_{s=1}^m \hat{\alpha}_{s,i} \mathbbm{1}\left[ k = y_i \right] - \hat{\alpha}_{k,i}$.
\end{theorem}
\begin{proof}
The proof of this theorem also leverages the stationary condition and the subdifferential of nuclear norm as in Theorem \ref{theo_binary_linear}. For the stationary condition of problem \eqref{multi_real_primal} with respect to parameter $A$,
\begin{align}
\partial \big \Vert \hat{A} \big \Vert_\ast - \sum_{i = 1}^n \sum_{k = 1}^m \hat{\alpha}'_{k,i} \Phi'_k\left( \mathbf{x}_i\right) = 0,
\end{align}
Then for $\hat{A} \in \mathbb{R}^{d_2 \times mp}$, consider its compact SVD $\hat{A} = \hat{U}_1\hat{D}_1\hat{V}_1^\top$, in which $\hat{U}_1 \in \mathbb{R}^{d_2 \times r}$, $\hat{D}_1 \in \mathbb{R}^{r \times r}$, and $\hat{V}_1 \in \mathbb{R}^{mp \times r}$, and the full SVD $\hat{A} = \hat{U}\hat{D}\hat{V}^\top$, in which $\hat{U} = [\hat{U}_1 \mid \hat{U}_2] \in \mathbb{R}^{d_2 \times d_2}$, $\hat{V} = [\hat{V}_1 \mid \hat{V}_2] \in \mathbb{R}^{mp \times mp}$. By Lemma \ref{subdifferntial}, we know that $\exists \ \hat{E} \in \mathbb{R}^{(d_2-r) \times (mp-r)}$ such that 
\begin{align}
\hat{U}_1\hat{V}_1^\top + \hat{U}_2\hat{E}\hat{V}_2^\top = \sum_{i = 1}^n \sum_{k = 1}^m \hat{\alpha}'_{k,i} \Phi'_k\left( \mathbf{x}_i\right). \label{multi_uv}
\end{align}
By computing the Gram matrix of both sides of Eq. \eqref{multi_uv},
\begin{align}
\left(\hat{U}_1\hat{V}_1^\top + \hat{U}_2\hat{E}\hat{V}_2^\top\right)^\top\left(\hat{U}_1\hat{V}_1^\top + \hat{U}_2\hat{E}\hat{V}_2^\top\right) &= \sum_{i = 1}^n \sum_{j = 1}^n \sum_{k = 1}^m \sum_{l = 1}^m \hat{\alpha}'_{k,i} \hat{\alpha}'_{l,j} \Phi'_k\left( \mathbf{x}_i\right)^\top \Phi'_l\left( \mathbf{x}_j\right) \nonumber \\
&= \sum_{i = 1}^n \sum_{j = 1}^n \sum_{k = 1}^m \hat{\alpha}'_{k,i} \hat{\alpha}'_{k,j} K'_k\left( \mathbf{x}_i, \mathbf{x}_j\right) \nonumber
\end{align}
where we generate the kernel generating matrix in the form of Eq. \eqref{multi_kernel} on the right side. For the left side, by a similar simplification procedure as in the proof of Theorem \ref{theo_binary_linear}, we have
\begin{align}
    \hat{V}_1\hat{V}_1^\top + \hat{V}_2\hat{E}^\top\hat{E}\hat{V}_2^\top = \sum_{i = 1}^n \sum_{j = 1}^n \sum_{k = 1}^m \hat{\alpha}'_{k,i} \hat{\alpha}'_{k,j} K'_k\left( \mathbf{x}_i, \mathbf{x}_j\right). \nonumber
\end{align}
Consequently, following the proof steps of Theorem \ref{theo_binary_linear}, we can say that $\hat{V}_1 = \big[\tilde{V}(i)\big]$ for all $i$ such that $\Lambda_{ii} = 1$ is the linear weight we recover, where we compute $\tilde{V}$ and $\Lambda$ by
\begin{align}
\tilde{V} \Lambda \tilde{V}^{-1} = \sum_{i = 1}^n \sum_{j = 1}^n \sum_{k = 1}^m \hat{\alpha}'_{k,i} \hat{\alpha}'_{k,j} K'_k\left( \mathbf{x}_i, \mathbf{x}_j\right). \nonumber
\end{align}
\end{proof}

\subsection{Proof of Theorem \ref{theom_multi_conv}}
\renewcommand{\thetheorem}{\ref{theom_multi_conv}}
\begin{theorem} 
We can implicitly recover the convolutional weight in the convolution output for multi-class classification. That is, for all $i \in [n]$, we have 
\begin{align}
    \Phi\left( \mathbf{x}_i \right)^\top\hat{U}_1 = \sum_{j = 1}^n \sum_{k=1}^m \alpha'_{k,j} \big [ \mathbf{0}_{p \times (k-1)p}, K \left(\mathbf{x}_i, \mathbf{x}_j\right), \mathbf{0}_{p \times (m-k)p} \big] \hat{V}_1,
\end{align}
where $\hat{\alpha}'_{k,i} = \sum_{s=1}^m \hat{\alpha}_{s,i} \mathbbm{1}\left[ k = y_i \right] - \hat{\alpha}_{k,i}$.
\end{theorem}
\begin{proof}
From Eq. \eqref{multi_uv} we know that 
\begin{align}
\hat{V}_1\hat{U}_1^\top =  \sum_{j = 1}^n \sum_{k = 1}^m \hat{\alpha}'_{k,j} \Phi'_k\left( \mathbf{x}_j\right)^\top - \hat{V}_2\hat{E}^\top\hat{U}_2^\top.
\end{align}
To form the convolution output on the left and the kernel generating matrix on the right, we multiply $\Phi\left(\mathbf{x}_i\right)$ on the right for both sides of the equation:
\begin{align} \nonumber
\hat{V}_1\hat{U}_1^\top \Phi\left( \mathbf{x}_i \right) &=  \sum_{j = 1}^n \sum_{k = 1}^m \hat{\alpha}'_{k,j} \Phi'_k\left( \mathbf{x}_j\right)^\top\Phi\left( \mathbf{x}_i \right) - \hat{V}_2\hat{E}^\top\hat{U}_2^\top \Phi\left( \mathbf{x}_i \right) \\ \nonumber
&= \sum_{j = 1}^n \sum_{k = 1}^m \hat{\alpha}'_{k,j} \left[\mathbf{0}_{d_2 \times (k-1)p},\Phi\left( \mathbf{x}_j\right), \mathbf{0}_{d_2 \times (m-k)p} \right]^\top \Phi\left( \mathbf{x}_i \right) - \hat{V}_2\hat{E}^\top\hat{U}_2^\top \Phi\left( \mathbf{x}_i \right)\\ \nonumber
&= \sum_{j = 1}^n \sum_{k = 1}^m \hat{\alpha}'_{k,j} \left[\mathbf{0}_{p \times (k-1)p},\left (\Phi\left( \mathbf{x}_j\right)^\top \Phi\left( \mathbf{x}_i \right) \right)^\top, \mathbf{0}_{p \times (m-k)p} \right]^\top - \hat{V}_2\hat{E}^\top\hat{U}_2^\top \Phi\left( \mathbf{x}_i \right) \\ 
&= \sum_{j = 1}^n \sum_{k = 1}^m \hat{\alpha}'_{k,j} \left[\mathbf{0}_{p \times (k-1)p},K\left( \mathbf{x}_j, \mathbf{x}_i \right)^\top, \mathbf{0}_{p \times (m-k)p} \right]^\top - \hat{V}_2\hat{E}^\top\hat{U}_2^\top \Phi\left( \mathbf{x}_i \right).
\end{align}
For both sides of the equation, multiply with $\hat{V}_1^\top$ on the left, we have
\begin{align}
\hat{V}_1^\top\hat{V}_1\hat{U}_1^\top \Phi\left( \mathbf{x}_i \right) &=  \sum_{j = 1}^n \sum_{k = 1}^m \hat{\alpha}'_{k,j} \hat{V}_1^\top \left[\mathbf{0}_{p \times (k-1)p}, K\left( \mathbf{x}_j, \mathbf{x}_i \right)^\top, \mathbf{0}_{p \times (m-k)p} \right]^\top \nonumber \\ 
& \quad + \hat{V}_1^\top\hat{V}_2\hat{E}^\top\hat{U}_2^\top \Phi\left( \mathbf{x}_i \right). \nonumber
\end{align}
Knowing that $\hat{V}_1^\top\hat{V}_1 = \mathbf{I}$ and $\hat{V}_1^\top\hat{V}_2 = \mathbf{0}$, we have
\begin{align}
\hat{U}_1^\top \Phi\left( \mathbf{x}_i \right) =  \sum_{j = 1}^n \sum_{k = 1}^m \hat{\alpha}'_{k,j} \hat{V}_1^\top \left[\mathbf{0}_{p \times (k-1)p},K\left( \mathbf{x}_j, \mathbf{x}_i \right)^\top, \mathbf{0}_{p \times (m-k)p} \right]^\top.
\end{align}
Taking transpose on both sides gives the form of the convolution output and completes the proof.
\end{proof}

\section{Fenchel conjugate of Common Losses} \label{app_conjugate}
\subsection{Hinge Loss} \label{hinge_conjugate}
The hinge loss function is given by $\ell_H(x)=\max\{0, 1-x\}$, $\forall x \in \mathbb{R}$. Its Fenchel conjugate is
\begin{align}
    \ell^\ast_H(x^\ast) = 
    \begin{cases}
    \ x^\ast \ , & x^\ast \in [-1, 0] \\
    \infty, & otherwise
    \end{cases}.
\end{align}
Derivation can be found in \cite{phdthesis}.

\subsection{Squared Hinge Loss} 
The squared hinge loss function is given by $\ell_{SH}(x)=(\max\{0, 1-x\})^2$, $\forall x \in \mathbb{R}$. Its Fenchel conjugate is
\begin{align}
    \ell^\ast_{SH}(x^\ast) = 
    \begin{cases}
    \ x^\ast + \frac{{x^\ast}^2}{4} \ , & x^\ast \leq 0 \\
    \infty, & otherwise
    \end{cases}.
\end{align}
Derivation can be found in \cite{phdthesis}.

\subsection{Logistic Loss} {
The logistic loss function is given by $\ell_{L}(x)=\log(1+e^x)$, $\forall x \in \mathbb{R}$. Its Fenchel conjugate is
\begin{align}
    \ell^\ast_{L}(x^\ast) = 
    \begin{cases}
    \ x^\ast \log(x^\ast) + (1-x^\ast)\log(x^\ast) \ , & x^\ast \in [0,1] \\
    \infty, & otherwise
    \end{cases}.
\end{align}
Derivation can be found in \cite{borwein2005convex}.
}
\subsection{Exponential Loss} {
The exponential loss function is given by $\ell_{E}(x)=e^x$, $\forall x \in \mathbb{R}$. Its Fenchel conjugate is
\begin{align}
    \ell^\ast_{E}(x^\ast) = 
    \begin{cases}
    \ -x^\ast + x^\ast \log(x^\ast) \ , & x^\ast \geq 0 \\
    \infty, & otherwise
    \end{cases}.
\end{align}
Derivation can be found in \cite{borwein2005convex}.
}

\section{Dual Optimization with Hinge Loss} \label{hinge_optimization}
We apply hinge loss for experimental evaluation. Given the dual optimization problem derived in Theorem \ref{dual_theorem}:
\begin{align} 
\maximize_{\alpha} 
\quad & 
-c\sum_{i = 1}^n \ell^\ast\left(-\frac{\alpha_i}{c}\right) \nonumber \\
\st \quad & 
\lambda_{max} \Big(  \sum_{i = 1}^n \sum_{j = 1}^n \alpha_i \alpha_j y_i y_j K(\mathbf{x}_i, \mathbf{x}_j) \Big) \leq 1 \,, \tag{\ref{real_dual}}  \\ 
& \alpha_i \geq 0 \,, 
\quad \forall i \in [n] \,, \nonumber  
\end{align}
and the Fenchel conjugate of the hinge loss function given in Appendix \ref{hinge_conjugate}, we know the objective function becomes
\begin{align}
    -c\sum_{i = 1}^n \ell^\ast\left(-\frac{\alpha_i}{c}\right) &= -c\sum_{i = 1}^n \left(-\frac{\alpha_i}{c}\right) \nonumber \\
    &= \sum_{i = 1}^n \alpha_i, \nonumber
\end{align}
for $-\frac{\alpha_i}{c} \in [-1, 0]$, i.e. $\forall i \in [n]$, $\alpha_i \in [0,c]$. Therefore the dual optimization problem with hinge loss is:
\begin{align} 
\maximize_{\alpha} 
\quad & 
\sum_{i = 1}^n \alpha_i \nonumber \\
\st \quad & 
\lambda_{max} \Big(  \sum_{i = 1}^n \sum_{j = 1}^n \alpha_i \alpha_j y_i y_j K(\mathbf{x}_i, \mathbf{x}_j) \Big) \leq 1 \,, \nonumber  \\ 
& 0 \leq \alpha_i \leq c \,, 
\quad \forall i \in [n] \,. \nonumber  
\end{align}

\section{Algorithms} \label{algorithms}
\subsection{Coordinate Descent Optimization} \label{cd_algorithm}

\begin{algorithm}[H]
\caption{Coordinate Descent Optimization of the Dual Problem \eqref{real_dual}}
\label{cd_alg}
\begin{algorithmic}
\State {\bfseries Input:} Data $\{(\mathbf{x}_i, y_i)\}_{i=1}^n$; Kernel function $\mathcal{K}$; Hyperparameter $c$.
\end{algorithmic}
\begin{algorithmic}[1]
    \For{$i$ = 1 to $n$} 
        \State compute $\lambda_{max}(K(\mathbf{x}_i, \mathbf{x}_i))$.
    \EndFor
    \State \textbf{end}
    \State Let $Indices$ be the sorted coordinates according to the ascending order of $\lambda_{max}(K(\mathbf{x}_i, \mathbf{x}_i))$
    \State Let $S = \varnothing$ be the set of the optimized coordinates.
    \State Let $R = \mathbf{0}_{p\times p}$.
    \For{$i$ in $Indices$}
        \State Compute $T = \sum_{j \in S} \hat{\alpha}_j y_iy_j\big(K(\mathbf{x}_i, \mathbf{x}_j) + K(\mathbf{x}_j, \mathbf{x}_i)\big)$
        \If{$\lambda_{max}\big(c^2K(\mathbf{x}_i, \mathbf{x}_i) + cT + R\big) \leq 1$}
            \State Let $\hat{\alpha}_i = c$. 
            \State Add $i$ to $S$.
        \Else
            \State Find the largest $\hat{\alpha}_i \in [0,c]$ such that $\lambda_{max}\big(\hat{\alpha}_i^2K(\mathbf{x}_i, \mathbf{x}_i) + \hat{\alpha}_iT + R \big) \leq 1$ using binary search. 
            \State Add $i$ to $S$.
        \EndIf
        \State \textbf{end}
        \State $R = R + \hat{\alpha}_i^2K(\mathbf{x}_i, \mathbf{x}_i) + \hat{\alpha}_iT$.
    \EndFor
    \State \textbf{end}
    \alglinenoPush{alg1}
\end{algorithmic}
\begin{algorithmic}
\State {\bfseries Output:} Dual solution $\{\hat{\alpha}_i\}_{i=1}^n$.
\end{algorithmic}
\end{algorithm}

\subsection{Training $\mathcal{D}$-layer DCCNNs} \label{training_algorithm}

\begin{algorithm}[H]
\caption{Training $\mathcal{D}$-layer DCCNNs}
\label{train_alg}
\begin{algorithmic}
\State {\bfseries Input:} Number of convolutional layers $\mathcal{D}$; Data $\{(\mathbf{x}_i, y_i)\}_{i=1}^n$; Kernel function $\mathcal{K}$. 
\end{algorithmic}
\begin{algorithmic}[1]
    \For{$\ell$ = 1 to $\mathcal{D}$}
        \State Construct dual problem in Eq. \eqref{real_dual} with $\{(\mathbf{x}_i, y_i)\}_{i=1}^n$ and $\mathcal{K}$
        \State Solve for the dual solution in the $\ell^{th}$ layer $\{\hat{\alpha}^\ell_i\}_{i=1}^n$.
        \State Recover the convolution output $\{\mathbf{x}_i^\ell\}_{i=1}^n$ and the linear weight $\hat{L}_{\ell}$ with Algorithm \ref{recover_alg}.
        \State Let $\mathbf{x}_i = \mathbf{x}_i^\ell$.
    \EndFor
    \State \textbf{end}
    \alglinenoPush{alg1}
\end{algorithmic}
\begin{algorithmic}
\State {\bfseries Output:} Dual solution of every layer $\{\{\hat{\alpha}_i^\ell\}_{i=1}^n\}_{\ell = 1}^{\mathcal{D}}$; Convolution output of every layer $\{\{\mathbf{x}_i^\ell\}_{i=1}^n\}_{\ell = 1}^{\mathcal{D}}$; Linear weight of the every layer $\{\hat{L}_{\mathcal{\ell}}\}_{\ell = 1}^\mathcal{D}$.
\end{algorithmic}
\end{algorithm}
\raggedbottom

\subsection{Making Predictions with $\mathcal{D}$-layer DCCNNs} \label{predict_algorithm}

\begin{algorithm}[H]
\caption{Making Predictions with $\mathcal{D}$-layer DCCNNs}
\label{pred_alg}
\begin{algorithmic}
\State {\bfseries Input:} Test data $\mathbf{x}_{new}$; Kernel function $\mathcal{K}$; Dual solution of every layer $\{\{\hat{\alpha}_i^\ell\}_{i=1}^n\}_{\ell = 1}^{\mathcal{D}}$; Convolution output of every layer $\{\{\mathbf{x}_i^\ell\}_{i=1}^n\}_{\ell = 1}^{\mathcal{D}}$; Linear weight of the every layer $\{\hat{L}_{\mathcal{\ell}}\}_{\ell = 1}^\mathcal{D}$.
\end{algorithmic}
\begin{algorithmic}[1]
    \For{$\ell$ = 1 to $\mathcal{D}-1$}
        \State Compute the convolution output for $\mathbf{x}_{new}$ by 
        \begin{align}
            \mathbf{x}'_{new} = \vecrize \big(\sum\nolimits_{j = 1}^n \hat{\alpha}^\ell_j y_j K(\mathbf{x}_{new}, \mathbf{x}_j) \hat{L}_\ell \big). \nonumber
        \end{align}  
        \State Let $\mathbf{x}_{new} = \mathbf{x}'_{new}$.
    \EndFor
    \State \textbf{end}
    \State In layer $\mathcal{D}$, compute the prediction by \vspace{-2pt} \Comment{Apply the final linear weight for classification} \begin{align}  
        f(\mathbf{x}_{new}) = \sign \big(\Tr (\sum\nolimits_{j = 1}^n \hat{\alpha}^\mathcal{D}_j y_j K(\mathbf{x}_{new}, \mathbf{x}_j) \hat{L}_\mathcal{D} \hat{L}_\mathcal{D}^\top)\big). \nonumber
    \end{align} 
    \alglinenoPush{alg1}
\end{algorithmic}
\begin{algorithmic}
\State {\bfseries Output:} Prediction $f(\mathbf{x}_{new})$.
\end{algorithmic}
\end{algorithm}

\section{Average Pooling Matrix} \label{avgpooling}
Convolutional output $\Phi\left( \mathbf{x}_i \right)^\top\hat{U}_1 \in \mathbb{R}^{p \times r}$ is formed by the output of $r$ filters, i.e. the convolution output has $r$ channels. Pooling is done for each channel, also known as a \emph{feature map}, of the convolutional output. Denote one column of $\Phi\left( \mathbf{x}_i \right)^\top\hat{U}_1$, i.e. one feature map as $\mathbf{x}'_{i,k} \in \mathbb{R}^{p}$ for $k \in [r]$. In an average pooling operation, a subset of entries, or a patch of the feature map is multiplied element-wise with an average pooling filter. The total number of pooling operations depends on the pooling filter width and stride. We use $q$ to represent the total number of pooling operation on one feature map, then we can generate an average pooling matrix $G \in \mathbb{R}^{q \times p}$. Assume the pooling filter has $b$ entries, then
\begin{align}
    G_{s,t} = 
    \begin{cases}
      \frac{1}{b} & \text{pooling filter multiplied with the $t^{th}$ feature map entry in the $s^{th}$ pooling operation}\\
      0 & \text{otherwise}
    \end{cases}. \nonumber
\end{align}
To perform average pooling on the convolution output, we can directly multiply the average pooling matrix on the left of the convolution output before vectorizing it and feeding it to the next layer, i.e.
\begin{align}
G\Phi\left( \mathbf{x}_i \right)^\top\hat{U}_1 = \sum_{j = 1}^n \hat{\alpha}_j y_j G K(\mathbf{x}_i,\mathbf{x}_j) \hat{V}_1. \nonumber
\end{align}
Then for the next layerwise training, $\{\vecrize\big(\sum_{j = 1}^n \hat{\alpha}_j y_j G K(\mathbf{x}_i,\mathbf{x}_j) \hat{V}_1\big)\}_{i=1}^n$ is regarded as the input. 

\section{Experiment} \label{experiment}
\subsection{Experiment Settings} \label{exp_setting}
\textbf{Data.} We use the MNIST \cite{726791} and the ImageNet datasets \cite{5206848} under the terms of the Creative Commons Attribution-Share Alike 3.0 license. In binary classification experiments, for both the MNIST and ImageNet datasets, we randomly pick two classes, images of digit 2 versus 3 from MNIST, images containing tench versus bathroom tissue from ImageNet, and use 2000 images for training, 100 images for validation, and 600 images for testing. We center crop the ImageNet images to $224 \times 224$, and transfer the pixel value to $[0,1]$. No other preprocessing is applied.

\textbf{Architecture.} (1) MNIST experiments: we implement 1-layer and 2-layer DCCNN for binary classification and 1-layer DCCNN for multiclass classification, and compare with its counterparts under the CCNN framework and the CNN trained with back-propagation SGD. We test all methods with filter width 5, stride 1, and padding size 2, and set the number of filters for CCNN and SGD-trained CNN to 16 for each layer. (2) ImageNet experiments: we follow the architectures similar to AlexNet \cite{krizhevsky2012imagenet} and VGG11 \cite{simonyan2014very} with some few modifications. For DCCNN and CCNN, we only count the number of convolutional layers, that is, 5 layers in AlexNet and 8 layers in VGG11, and concatenate one final linear layer for the purpose of classification. In order to reduce the number of patches for computational efficiency, we change the stride of the first layer from 4 to 8 for both AlexNet and VGG11. To ensure VGG11 keeps all information of the input image, we correspondingly change the filter width in the first layer from 3 to 9, and remove the pooling after layer 4 and 6. For the simplicity of matrix multiplication, we choose average pooling as the pooling operation. Other settings of filter width, stride and padding size are kept the same. All methods are trained with the same architecture.

\textbf{Kernel function and activation function.} For both DCCNN and CCNN, we choose the Gaussian RBF kernel as in \cite{zhang2017convexified} that has the form $\mathcal{K}(\mathbf{z}_i, \mathbf{z}_j) = \exp\{-\gamma\Vert\mathbf{z}_i - \mathbf{z}_j\Vert^2_2\}$, $\Vert\mathbf{z}_i\Vert_2 = \Vert\mathbf{z}_j\Vert_2 = 1$. For a fair comparison, we choose the sinusoid function $\rho(\cdot) = \sin(\cdot)$ \cite{Isa2010, sopena1999neural} as the activation function, as it is proved in \cite{zhang2017convexified} to be contained in the RKHS of Gaussian RBF kernel.

\textbf{CCNN kernel matrix factorization.} For the kernel factorization $K = QQ^\top$ in the algorithm of \cite{zhang2017convexified}, we applied the following 4 approaches: (1) $Q = UD^{\frac{1}{2}}$ for $K = UDU^\top$; (2) $Q = UD^{\frac{1}{2}}V$ for $K = UDU^\top$, in which $V$ is a random orthonormal matrix; (3) $Q = K^{\frac{1}{2}}$; (4) $Q$ from the Cholesky Decomposition of $K$. For the consideration of computational complexity, we only construct the block diagonal of the kernel matrix $K$, and approximate the sample feature by the factorization of its corresponding block diagonal.

\textbf{Loss function.} For CCNN and DCCNN, we use the hinge loss for optimization, and for the nuclear norm constraint in CCNN, we implement it as regularization in the optimization step. For back-propagation SGD, we use the binary cross-entropy loss for optimization.

\textbf{Hyperparameters.}
For CNNs trained with SGD, we use 50 epochs with batch size 50 and learning rate 0.1. For a fair comparison, we do not employ dropout, weight decay, or other tricks for training. We run 10 trials for each experiment and report the average and standard deviation of accuracy. For CCNN, since the block diagonal kernel matrix is much smaller than the full kernel matrix, we set the CCNN hyperparameters  $m=25$ and $r=16$ for each layer. For DCCNN, even though we theoretically take the eigenvectors with eigenvalue 1, there could be numerical issues during the optimization process, making the eigenvalues not strictly 1. Therefore we set a threshold for the eigenvalues we take. For 1-layer DCCNN, 2-layer DCCNN, AlexNet DCCNN, and VGG11 DCCNN, the threshold is set to 0.8, 0.9, 0.975, 0.85, respectively.

\textbf{Platform and implementation.}
We implement DCCNN with Matlab 2018a. The code is tested on a server with 4 CPU cores and 16GB memory size.

\subsection{Results} \label{sgd_sdv}
We demonstrate the results of prediction accuracy in Table \ref{exp_res}. The rows are organized by the datasets and architectures while each column shows the performance of one method across different tasks.

For SGD experiments, in particular, we run 10 trials for each setting with different random seeds and report the average and standard deviation of accuracy. The average accuracy is shown in Table \ref{exp_res}. For binary classification, we observe a $1.8\%$ standard deviation for one-layer CNN trained with SGD on the MNIST dataset, and $0.6\%$ for the two-layer CNN trained end-to-end with SGD. For the network structure similar to AlexNet on the ImageNet dataset, we observe a standard deviation of $2.5\%$, and for the architecture similar to VGG11, the standard deviation is $1.7\%$. For multiclass classification, the standard deviation of SGD is $0.5\%$.

\begin{table}[tp]
 \begin{subtable}[h]{\textwidth}
  \centering
  \begin{tabular}{llllllll}
    \toprule
    \multirow{2}{*}{Dataset}  & \multirow{2}{*}{Architecture} & \multicolumn{4}{c}{CCNN} &  \multirow{2}{*}{SGD} & \multirow{2}{*}{DCCNN} \\
    \cmidrule(lr){3-6}
    & & $UD^{\frac{1}{2}}$ & $UD^{\frac{1}{2}}V$ & $K^{\frac{1}{2}}$ & Cholesky &  &  \\
    \midrule
    \multirow{2}{*}{MNIST}  & 1-Layer & 90.3\% & 90.3\% & 88.8\% & 93.7\% & 90.2\% & 94.8\%\\
    \cmidrule(lr){3-8}
                            & 2-Layer & 93.7\% & 93.5\% & 93.5\% & 96.7\% & 95.2\% & 96.0\% \\
     \bottomrule
     \toprule
    \multirow{2}{*}{ImageNet} & AlexNet & 62.3\% & 62.3\% & 53.2\% & 62.0\% & 87.1\% & 83.3\%  \\
    \cmidrule(lr){3-8}
            & VGG11 & 55.3\% & 52.5\% & 59.8\% & 57.0\% & 89.2\% & 85.0\% \\
    \bottomrule
  \end{tabular}
  \vspace{5pt}
  \caption{Results on MNIST and ImageNet Binary Classification. }
  \vspace{0pt}
 \end{subtable}

 \vspace{10pt}
 
 \begin{subtable}[h]{\textwidth}
  \centering
  \begin{tabular}{llllllll}
    \toprule
    \multirow{2}{*}{Dataset} & \multirow{2}{*}{Architecture} & \multicolumn{4}{c}{CCNN} &  \multirow{2}{*}{SGD} & \multirow{2}{*}{DCCNN} \\
    \cmidrule(lr){3-6}
    & & $UD^{\frac{1}{2}}$ & $UD^{\frac{1}{2}}V$ & $K^{\frac{1}{2}}$ & Cholesky &  &  \\
    \midrule
    {MNIST}  &  1-Layer & 75.4\% & 75.4\% & 82.3\% & 85.9\% & 87.0\% & 85.3\%\\
     \bottomrule
  \end{tabular}
  \vspace{5pt}
  \caption{Results on MNIST Multiclass Classification.}
  \vspace{0pt}
 \end{subtable}
 \renewcommand\thetable{1}
 \caption{Experiment results for binary and multiclass classification. Methods compared include CCNN, CNN trained with SGD, and our proposed DCCNN. For CCNN we include 4 different ways for kernel matrix factorization: (1) $Q = UD^{\frac{1}{2}}$ for $K = UDU^\top$; (2) $Q = UD^{\frac{1}{2}}V$ for $K = UDU^\top$, and random orthonormal matrix $V$; (3) $Q = K^{\frac{1}{2}}$; (4) $Q$ from the Cholesky Decomposition of $K$.}
\end{table}

\subsection{Discussion on Performance Level} \label{exp_discussion}
In the binary classification task of the MNIST data, we can see that DCCNN outperforms CNN optimized by SGD and all different kernel matrix factorizations for CCNN on one-layer and two-layer networks with only one exception for the two-layer CCNN with Cholesky Decomposition. This verifies the effectiveness of DCCNN. Furthermore, we observe that different factorization approaches introduce turbulence to the CCNN method, while our algorithm does not suffer from such ambiguity.  

On the more complicated ImageNet dataset, DCCNN also performs comparably well with the end-to-end SGD optimized CNNs under both AlexNet and VGG11 architectures, and significantly outperforms the CCNN method. With the heuristic of kernel matrix factorization and cutting the weight matrix, CCNN does not necessarily generalize to complex datasets like ImageNet. As the task gets difficult, the performance level of CCNN with different factorization methods gets more arbitrary. This further highlights the merits of our proposed DCCNN.

Moreover, we observe that there is a decrease of accuracy for CCNN with the increase on the number of layers. That may be caused by the accumulated impact of the weight-cutting heuristic, i.e., the number of filters heuristically enforced by a hyperparameter in each layer may not accurately reflect the information learned. Setting it too large would include unnecessary noise while setting it too small may discard relevant information. This further highlights the effectiveness of DCCNN on encouraging a small number of filters without introducing heuristics or ambiguity.

\subsection{Discussion on Computational Complexity}
We first analyze theoretically that
\begin{itemize}
    \item spatially, CCNN has to construct the whole kernel matrix of size $np \times np$ with all of the $n$ samples and $p$ patches, leading to spatial complexity of $\mathcal{O}(n^2p^2)$, while DCCNN only need the kernel generating matrix of one sample at a time, which is of size $p \times p$, leading to the cost of $\mathcal{O}(p^2)$ space;
    \item temporally, the factorization of the kernel matrix in CCNN takes $\mathcal{O}(n^3p^3)$, while the runtime of DCCNN is dominated by the eigendecomposition in Algorithm \ref{cd_alg}, leading to time complexity of $\mathcal{O}(n^2p^3)$.
\end{itemize}
Experimentally, we now demonstrate how the running time of each method scales with the number of samples $n$ and the number of patches $p$ for the following reasons: (1) It takes an extremely long time to run the full kernel matrix (size $np \times np$) factorization for CCNN without approximation tricks, thus making such comparisons on actual running time trivial. So for CCNN baseline we only construct the kernel matrix of a batch of samples as applied in the CCNN paper \cite{zhang2017convexified} to accelerate CCNN. (2) Methods are implemented differently, e.g. MATLAB v.s. PyTorch, hand-crafted matrix multiplications v.s. built-in conv \& pooling operations, etc.   

For 2000 samples 784 patches and 200 samples and 100 patches on 1-Layer MNIST experiment, the running time ratio is shown in Table \ref{runtime_ratio}:

\begin{table}[h]
  \centering
  \begin{tabular}{cccc}
    \toprule
    1-Layer MNIST running time  &  DCCNN  & CCNN & SGD  \\
    \midrule
    $(n,p)=(2000,784)$ / $(n,p)=(200,100)$ & 580x & 1300x & 55x \\
    \bottomrule
  \end{tabular}
  \vspace{5pt}
  \renewcommand\thetable{2}
  \caption{Running time ratio between data and models of different scales.}
  \label{runtime_ratio}
  \vspace{0pt}
\end{table}

As we see from the result, DCCNN is more computationally efficient than CCNN. Though SGD runs even faster, we emphasize that SGD is a stochastic approximation algorithm in nature. 

Furthermore, the improvement on space complexity originates naturally from the dual formulation of the problem without the kernel matrix construction, which once again shows the significance of DCCNN. On the other hand, time complexity relies more on the solving algorithm, which is only an initial design in DCCNN as we focus more on the formulation of the dual problem and weight recovery from the optimized dual solution, and may be improved with better engineering.

\section{Further Discussion on Related Works} \label{morerelatedwork}
On learning convolutional neural networks with kernel methods, following the work of Neural Tangent Kernel (NTK) \cite{jacot2018neural}, \cite{arora2019exact} proposes the Convolutional Neural Tangent Kernel (CNTK) that studies the exact computation of CNNs with infinitely many convolutional filters, i.e. infinitely wide CNNs. As our approach uses kernel information and implies an infinite size of the convolutional weight matrix, it is fundamentally different from CNTK in both formulation and derivation: (1) In CCNN formulation described in Eq. \eqref{DCCNNeq}, the convolutional weight is $W \in \mathbb{R}^{d_2 \times r}$. $r$ is the number of filters, which is not only finite but encouraged to be low by the nuclear norm constraint, while CNTK studies the case of infinite many filters. In CCNN formulation, $d_2$, the size of each filter, is the part that may go to infinity, while CNTK assumes each filter has a fixed size. Therefore, the CCNN and CNTK frameworks are fundamentally different.
(2) CNTK is to take the infinite limit of width so that the inputs of activation function tend to i.i.d. centered Gaussian processes with fixed covariance \cite{jacot2018neural, arora2019exact}, under which condition the output of the neural network would converge to the output of CNTK, asymptotically (Theorem 1 in \cite{jacot2018neural}) or non-asymptotically (Theorem 3.2 in \cite{arora2019exact}). On the other hand, the infinity size of the weight matrix in CCNN or DCCNN comes from the infinite dimension of the kernel basis function. 

On deriving the convex equivalence of CNNs, \cite{ergen2020implicit} derives a convex analytic framework utilizing semi-infinite duality, and regards the CNN architecture as an implicit convex regularizer following previous work \cite{pilanci2020neural}. As our approach utilizes the power of duality, it is not for the purpose of convexifying CNNs, but to eliminate the need and the ambiguity of factorizing the very large kernel matrix in CCNN \cite{zhang2017convexified}. As a result, the optimization problem is completely different. The most significant difference, the nuclear norm regularization in DCCNN poses challenges as it is non-differentiable, leading to no closed-form solution for recovering the primal solution from the dual. For this particular challenge, we proposed a highly-novel weight recovery algorithm to recover the linear weight and the output of the convolutional layer directly, instead of the convolutional weight.